\documentclass{article}

\pdfoutput=1

\usepackage{microtype}
\usepackage{graphicx}
\usepackage{subcaption}
\usepackage{booktabs} % for professional tables
\usepackage{stfloats}
\usepackage{multirow}
\usepackage{hyperref}

\usepackage[accepted]{icml2026}

\usepackage{amsmath}
\usepackage{amssymb}
\usepackage{mathtools}
\usepackage{amsthm}

\usepackage[capitalize,noabbrev]{cleveref}

\theoremstyle{plain}
\newtheorem{theorem}{Theorem}[section]

\theoremstyle{definition}

\theoremstyle{remark}

\usepackage[disable,textsize=tiny]{todonotes}

\usepackage[utf8]{inputenc}
\usepackage{amsmath, amssymb, amsthm}
\usepackage{graphicx}
\usepackage{hyperref}
\usepackage{xcolor}
\usepackage{tcolorbox}
\tcbuselibrary{skins,breakable}
\usepackage{pifont} % for \ding symbols

\newtcolorbox{examplebox}[1]{
  colback=gray!5!white,
  colframe=gray!50!black,
  fonttitle=\bfseries,
  title={#1},
  breakable,
  enhanced,
  left=2mm, right=2mm, top=1mm, bottom=1mm
}
\newtcolorbox{failurebox}[1][]{
  colback=red!5!white,
  colframe=red!50!black,
  fonttitle=\bfseries,
  title=#1,
  breakable,
  enhanced,
  left=2mm, right=2mm, top=1mm, bottom=1mm
}

\newtcolorbox{pathbox}{
  colback=gray!5!white,
  colframe=gray!50!black,
  boxrule=0.5pt,
  left=2mm, right=2mm, top=1mm, bottom=1mm,
  fontupper=\small\ttfamily
}
\usepackage{algorithm}
\usepackage{colortbl}
\usepackage{xcolor}

\icmltitlerunning{RSF-GLLM: Bridging the Semantic Gap in Multi-Hop Knowledge Graph QA}

\begin{document}

\twocolumn[
  \icmltitle{RSF-GLLM: Bridging the Semantic Gap in Multi-Hop Knowledge Graph QA via Recurrent Soft-Flow and Decoupled LLM Generation}



  \icmlsetsymbol{equal}{*}

  \begin{icmlauthorlist}
    \icmlauthor{Sambaran Bandyopadhyay}{adres}
    \icmlauthor{Ananth Muppidi}{adsys}
  \end{icmlauthorlist}

  \icmlaffiliation{adres}{Adobe Research}
  \icmlaffiliation{adsys}{Adobe Systems}
%   \icmlaffiliation{sch}{School of ZZZ, Institute of WWW, Location, Country}

  \icmlcorrespondingauthor{Sambaran Bandyopadhyay}{samb.bandyo@gmail.com}
%   \icmlcorrespondingauthor{Ananth Muppidi}{ananth.muppidi@gmail.com}

  \icmlkeywords{Machine Learning, ICML}

  \vskip 0.3in
]



\printAffiliationsAndNotice{}  % no special notice (required even if empty)

\begin{abstract}
Multi-hop Question Answering over Knowledge Graphs faces a critical challenge: traditional retrieve-then-read pipelines break differentiability, preventing the retriever from learning to bridge the semantic gap where intermediate nodes lack lexical overlap with the query. To address this, we propose RSF-GLLM, a framework decoupling differentiable graph reasoning from answer generation. Our Recurrent Soft-Flow (RSF) module employs a GRU-guided query updater to propagate continuous relevance scores, utilizing a dynamic gating mechanism to traverse semantically dissimilar bridge nodes via structural cues. We introduce flow sparsity regularization to theoretically guarantee convergence from soft probabilities to discrete reasoning paths. These paths are extracted and textualized to fine-tune a Large Language Model (LLM), ensuring generation is grounded in factual topology. Experiments on WebQSP and CWQ demonstrate that RSF-GLLM achieves competitive performance with superior inference efficiency compared to LLM based computationally expensive approaches.
\end{abstract}

\section{Introduction}
\label{sec:intro}

Reasoning over structured knowledge bases to answer complex, multi-hop queries is a fundamental challenge in artificial intelligence, essential for applications requiring high factual precision such as biomedical discovery and financial auditing \cite{wang2021survey}. While Large Language Models (LLMs) have demonstrated exceptional fluency in open-domain tasks \cite{brown2020language}, they frequently exhibit hallucinations and unfaithful reasoning when operating over long-tail knowledge or complex logical chains \cite{ji2023survey, nandy2025language}. Consequently, Knowledge Graphs (KGs) remain indispensable for grounding generation in verifiable facts \cite{hogan2021knowledge, nandy2025language}. However, effective retrieval over KGs is impeded by the \textit{semantic gap}. Consider the query: \textit{``What awards did the director of Inception win?''}. To answer this, a system must traverse from the topic entity \textit{``Inception''} to the intermediate node \textit{``Christopher Nolan''} (the director), and finally to \textit{``Academy Award''}. Crucially, the intermediate bridge entity \textit{``Christopher Nolan''} shares no lexical overlap with the query terms ``awards'' or ``win'', rendering standard similarity-based traversal ineffective \cite{sun2018open, xiong2017deeppath}.

Current approaches to Knowledge Graph Question Answering (KGQA) effectively bifurcate into two paradigms, each utilizing distinct inductive biases yet suffering from complementary limitations. The first category comprises \textit{iterative subgraph retrieval} methods, such as GraftNet \cite{sun2018open}, PullNet \cite{sun2019pullnet}, and NSM \cite{he2021improving}. These architectures typically treat reasoning as a discrete \textit{Retrieve-then-Read} process \cite{karpukhin2020dense}. However, the discrete selection of nodes at each hop breaks end-to-end differentiability, preventing the retriever from adapting to downstream errors \cite{kim2023tree}. Furthermore, rigid entity-linking requirements often cause these models to fail when the requisite \textit{bridge} nodes are semantically dissimilar to the query text \cite{liang2024survey}.

The second, more recent paradigm integrates the \textit{generative capabilities of LLMs} directly. Frameworks such as Graph of Thoughts \cite{besta2024graph} and RoG \cite{luo2024reasoning} attempt to improve expressivity by structuring LLM reasoning as an iterative graph traversal. However, these agentic approaches incur prohibitive computational costs, often requiring multiple passes through billion-parameter models to answer a single query \cite{sun2024think}. Furthermore, while methods like Self-RAG \cite{asai2024selfrag} and CRAG \cite{yan2024corrective} have introduced critique-based loops to detect hallucinations, they still struggle with the \textit{reasoning shortcut} problem: the generator frequently ignores retrieved structural constraints in favor of its pre-trained parametric memory, compromising factual grounding \cite{li2024chain, gao2023retrieval}.

To reconcile the structural rigor of differentiable traversal with the efficiency required for scalable deployment, we propose \textbf{RSF-GLLM} (Recurrent Soft-Flow Graph-to-LLM). Unlike monolithic end-to-end architectures, our framework decouples reasoning from generation. We first introduce a \textit{Recurrent Soft-Flow (RSF)} module, a lightweight differentiable reasoned that propagates continuous probability mass guided by a Recurrent Query Updater \cite{hudson2019learning}. To bridge the semantic gap, we devise a \textit{Dynamic Gating Mechanism} that modulates the influence of node content versus graph topology, allowing the model to traverse semantically disjoint nodes purely via valid structural relations. Crucially, we impose a \textit{Flow Sparsity Regularization} \cite{jang2016categorical} to force these continuous distributions to converge to discrete reasoning paths, ensuring interpretability.

Following are the \textbf{contributions} we made in this paper:

\textit{Semantic Gap in Differentiable KGQA:}
We identify and formalize a critical failure mode in existing differentiable multi-hop KGQA models: the semantic gap, where enforcing semantic similarity at intermediate hops suppresses valid reasoning paths through semantically disjoint bridge entities.

\textit{Dynamic Structure–Semantics Decoupling:}
We propose Recurrent Soft-Flow (RSF), a differentiable graph reasoning module that introduces a dynamic gating mechanism to adaptively decouple structural propagation from semantic matching. This enables the model to rely purely on graph topology when semantic signals are misleading, and to re-introduce content bias only when disambiguation is required.

\textit{Principled Sparsity for Interpretable and Faithful Reasoning Paths:}
We introduce an entropy-based flow sparsity regularization and provide a theoretical analysis showing that it incentivizes convergence from soft relevance distributions to discrete, interpretable reasoning paths, enabling \textit{provably} faithful path extraction without sacrificing differentiability.

\textit{Decoupled Graph-to-LLM Architecture for Efficient Grounded QA:}
We propose a two-stage Graph-to-LLM framework that cleanly decouples graph reasoning from answer generation. A lightweight differentiable graph module performs all structural inference, while a large language model is used solely as a grounded conditional generator over extracted reasoning paths—avoiding expensive agentic reasoning loops.

\textit{Empirical Validation Across Accuracy, Robustness, and Efficiency:}
Extensive experiments on WebQSP and CWQ demonstrate that RSF-GLLM achieves competitive accuracy, robust performance on semantic-gap-heavy queries, and orders-of-magnitude inference efficiency improvements over recent LLM-centric and agentic KGQA baselines.

% \textbf{Decoupled Efficiency:} We propose a two-stage architecture where a lightweight differentiable module extracts reasoning paths for a grounded generator, reducing inference latency by orders of magnitude compared to LLM-centric baselines \cite{luo2024reasoning}.
% \\
% \textbf{Semantic Gap Bridging:} We introduce a Dynamic Gating Mechanism and theoretically demonstrate its capacity to prioritize structural validity over semantic matching when necessary, addressing a key failure mode of dense retrievers.
% \\
% \textbf{Theoretical Sparsity:} We provide a theoretical analysis proving that our entropy-based regularization forces soft-flow convergence to discrete paths, enabling the extraction of crisp reasoning chains from differentiable operations.
% \\
% \textbf{SOTA Performance:} Extensive experiments on WebQSP \cite{yih2016value} and CWQ \cite{talmor2018web} demonstrate that RSF-GLLM achieves state-of-the-art performance, effectively balancing cost, accuracy, and grounding.

\begin{figure*}[htbp]
    \centering
    \includegraphics[width=0.9\textwidth]{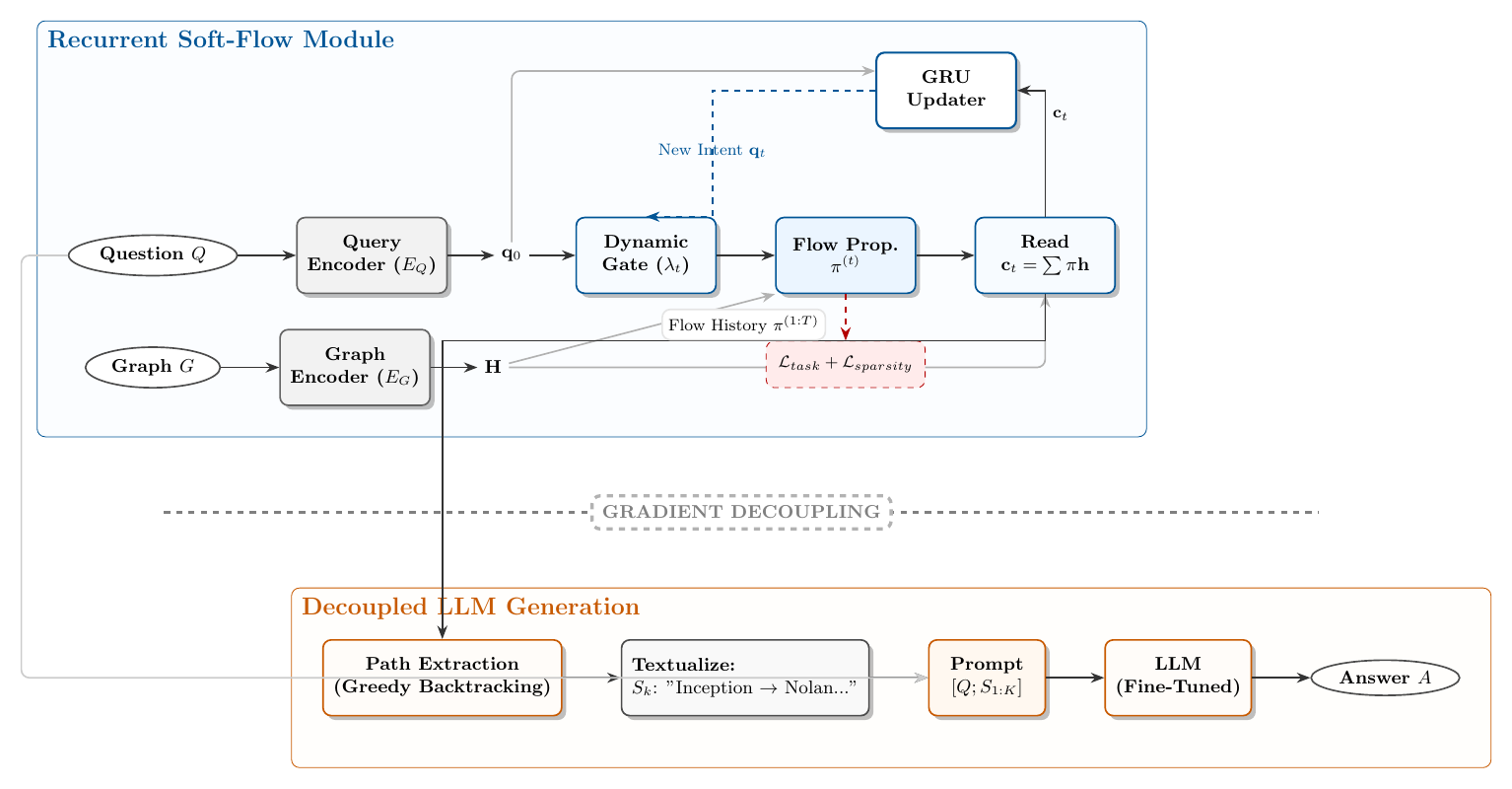}
    \caption{Overall flow diagram of RSF-GLLM}
    \label{fig:flow_diag}
\end{figure*}

\section{Methodology}

\subsection{Problem Formulation}
We formulate Multi-Hop Question Answering (MHQA) over Knowledge Graphs as a path-grounded generation task. Let $\mathcal{G} = (\mathcal{V}, \mathcal{R}, \mathcal{E})$ be a Knowledge Graph, where $\mathcal{V}$ is the set of entities, $\mathcal{R}$ is the set of relation types, and $\mathcal{E} \subseteq \mathcal{V} \times \mathcal{R} \times \mathcal{V}$ is the set of factual triplets. Given a natural language query $Q$, the objective is to generate an answer $A$ derived from a reasoning chain within $\mathcal{G}$.

Unlike standard retrieval-augmented generation (RAG) which retrieves unstructured text passages, we aim to discover a structural reasoning path $P = (v_0, r_1, v_1, \dots, r_T, v_T)$ with $v_0$ as a topic entity identified in $Q$, such that:
\begin{enumerate}
    \item Each step $(v_{t-1}, r_t, v_t) \in \mathcal{E}$ is a valid triplet in $\mathcal{G}$.
    \item The path $P$ semantically bridges the gap between the query intent and the answer node $v_T$.
\end{enumerate}
The final answer $A$ is generated by maximizing the probability $P(A \mid Q, P)$, ensuring the generation is hallucination-free and strictly grounded in the graph topology. We address the core challenge where intermediate nodes $v_t$ (for $0 < t < T$) may share no lexical similarity with $Q$ (the \textit{Semantic Gap}), requiring a differentiable traversal mechanism that learns to propagate probability mass $P(v_T \mid Q, \mathcal{G})$ via structural cues rather than surface matching.

Now, we discuss the details of our proposed solution \textbf{RSF-GLLM}. As shown in Figure \ref{fig:flow_diag}, it operates in a decoupled two-stage manner: first, we optimize a differentiable graph traverser (RSF Module) to identify reasoning paths; second, we fine-tune an LLM to generate answers conditioned on the textualized paths retrieved by the RSF module.

\subsection{Phase 1: Anchor Identification \& Initialization}
Since processing the entire KG is intractable, we first identify a computational workspace. Unlike older methods that rely on exact string matching, we use dense retrieval to handle synonyms and ambiguity.

\subsubsection{Dense Anchor Retrieval}
We employ a bi-encoder architecture (similar to DPR \cite{karpukhin2020dense}). We pre-compute embeddings for all entities $e \in \mathcal{V}$ using a BERT-based encoder $E_{Ent}$. Given question $Q$, we encode it into $\mathbf{v}_Q = E_{Q}(Q)$. We retrieve the topic entity $v_{topic}$ via Maximum Inner Product Search (MIPS):
\begin{equation}\label{eq:anchor_retr}
    v_{topic} = \underset{e \in \mathcal{V}}{\text{argmax}} (\mathbf{v}_Q^\top E_{Ent}(e))
\end{equation}
\textit{Example:} For $Q=$ \textit{``What awards did the director of Inception win?''}, the model retrieves $v_{topic} = \mathrm{Inception}$. We utilize a FAISS index \cite{johnson2019billion} for efficient retrieval.

\subsubsection{Subgraph Construction}
We extract the $T$-hop neighborhood around $v_{topic}$ to form the subgraph $\mathcal{G}_{sub} = (\mathcal{V}_{sub}, \mathcal{E}_{sub})$. To ensure alignment between the text question and graph nodes, we re-initialize all node features in $\mathcal{G}_{sub}$ using the \textbf{same} encoder used for the question.
\begin{equation}
    \mathbf{h}_v^{(0)} = \mathbf{W}_{\mathcal{V}} \cdot E_{Q}(\text{text}(v)), \quad \mathbf{e}_r = \mathbf{W}_{\mathcal{R}} \cdot E_{Q}(\text{text}(r))
\end{equation}
where $\mathbf{W}_{\mathcal{V}}, \mathbf{W}_{\mathcal{R}}$ are learnable projection matrices.

\subsection{Phase 2: Recurrent Graph Reasoning (The RSF Module)}
This module performs $T$ steps of differentiable reasoning ($t=1 \dots T$). At each step, the model updates a probability distribution over nodes (The Flow) \cite{sun2018open} and evolves the question vector (The Intent) \cite{hudson2019learning}.

\textbf{Initialization:} The flow $\boldsymbol{\pi}^{(0)}$ is a one-hot vector at $v_{topic}$. The query state is $\mathbf{q}_0 = \mathbf{v}_Q$.

\subsubsection{Step 2.1: Dynamic Relation Attention (The Compass)}
At step $t$, the model decides which edge types are valid. We compute an attention vector $\boldsymbol{\beta}^{(t)}$ over all relations using a Sigmoid activation to allow multi-path traversal:
\begin{equation}\label{eq:relatt}
    \beta_r^{(t)} = \sigma \left( \mathbf{q}_{t-1}^\top \mathbf{W}_{att} \mathbf{e}_r \right)
\end{equation}
\textit{Example:} $\beta_{\texttt{directed\_by}}^{(1)} \approx 1.0$, while $\beta_{\texttt{released\_in}}^{(1)} \approx 0.0$.

\subsubsection{Step 2.2: Relation-Guided Flow Propagation}
This step updates the node relevance distribution by propagating flow from the previous hop. We decompose this process into structural propagation and conditional semantic refinement.

\textbf{1. Structural Flow Propagation.}
Following standard message-passing neural networks (MPNNs) \cite{gilmer2017neural} and previous KGQA frameworks like GraftNet \cite{sun2018open}, we first calculate the probability mass arriving at a node $v$ purely based on graph connectivity and relation validity:
\begin{equation}
    \Phi_v^{(t)} = \sum_{(u,r,v) \in \mathcal{N}_{in}(v)} \pi_u^{(t-1)} \cdot \beta_r^{(t)}
\end{equation}
where $\mathcal{N}_{in}(v)$ denotes the incoming neighborhood. $\Phi_v^{(t)}$ represents the structural likelihood of $v$ being the next step in the reasoning chain.

\textbf{2. Semantic Disambiguation via Content Bias.}
Structural flow alone suffers from the \textit{fan-out dilution} problem \cite{sun2019pullnet} when multiple neighbors are connected via the same valid relation. Consider a modified query: \textit{``Did the director of Inception also direct Interstellar?''}. At hop $t=2$, the relation \texttt{directed\_by}$^{-1}$ is valid for all movies directed by Christopher Nolan. Without node content awareness, the structural flow $\Phi_v^{(t)}$ is uniformly diluted across \textit{Interstellar}, \textit{Dunkirk}, and \textit{Tenet}, making the correct target indistinguishable from noise.
To resolve this, we introduce a content matching score $\rho_v^{(t)}$:
\begin{equation}
    \rho_v^{(t)} = \mathbf{q}_{t-1}^\top \mathbf{W}_{node} \mathbf{h}_v
\end{equation}
This score measures the alignment between the node and the current query intent, allowing the model to highlight ``Interstellar'' specifically.

\textbf{3. Adaptive Gating Mechanism.}
Indiscriminately applying content bias is detrimental due to the \textbf{Semantic Gap}: intermediate reasoning nodes often share no lexical overlap with the query. In our running example \textit{``What awards\dots''}, the intermediate node ``Christopher Nolan'' does not semantically match the query word ``Awards.'' Enforcing content bias here would erroneously suppress the correct bridge node.
We therefore employ a dynamic gating mechanism to modulate the influence of the content bias:
\begin{equation}\label{eq:gatingLambda}
    \lambda^{(t)} = \sigma(\mathbf{w}_{\lambda}^\top \mathbf{q}_{t-1} + b_\lambda)
\end{equation}
Here, $\lambda^{(t)} \to 1$ when the query seeks specific named entities (resolving fan-out), and $\lambda^{(t)} \to 0$ when the query focuses on structural traversal (bridging the semantic gap).

\textbf{4. Probabilistic Aggregation.}
We combine the structural and semantic signals into a unified probability distribution. The aggregation models a \textit{conditional} conjunction: a node must be structurally valid, and—\textit{only if} the gate $\lambda^{(t)}$ is active—semantically relevant. We implement this via summation in log-space:
\begin{equation}\label{eq:probAgg}
    \pi_v^{(t)} = \frac{\exp \left( \log(\Phi_v^{(t)} + \epsilon) + \lambda^{(t)} \cdot \rho_v^{(t)} \right)}{\sum_{z \in \mathcal{V}_{sub}} \exp \left( \log(\Phi_z^{(t)} + \epsilon) + \lambda^{(t)} \cdot \rho_z^{(t)} \right)}
\end{equation}
where $\epsilon$ is a smoothing constant. When $\lambda^{(t)} \approx 0$, the term $\lambda^{(t)} \cdot \rho_v^{(t)}$ vanishes, and the selection relies purely on structural validity $\Phi_v^{(t)}$. When $\lambda^{(t)} \approx 1$, the selection requires both structure and content alignment.

\textbf{Remark on Re-entrant Paths.}
Contrary to acyclic traversal constraints often imposed in symbolic reasoning, RSF-GLLM explicitly permits re-entrant paths where flow returns to a previously visited node. This is essential for solving \textit{circular filtering} queries, such as \textit{``Which movies starring Leonardo DiCaprio were directed by the director of Inception?''}. To answer this, the reasoning path must traverse $Inception \xrightarrow{directed\_by} Nolan \xrightarrow{directed} Inception$, returning to the start node to verify the second constraint (``starring DiCaprio''). Strict acyclicity would discard the correct answer ($Inception$) at step $t=2$. Instead of hard blocking, we rely on the \textbf{Query Update} mechanism (Section \ref{sec:query_update}) to prevent redundant oscillations.

\subsubsection{Step 2.3: Context Aggregation (The \textit{Read} Operation)}
Once the flow distribution $\boldsymbol{\pi}^{(t)}$ is computed, we summarize the retrieved information into a single differentiable context vector $\mathbf{c}_t$. This operation functions as a \textit{soft read} from the graph memory.
We compute $\mathbf{c}_t$ as the expected value of the node embeddings under the current flow distribution:
\begin{equation}
    \mathbf{c}_t = \sum_{v \in \mathcal{V}_{sub}} \pi_v^{(t)} \cdot \mathbf{h}_v
\end{equation}
where $\mathbf{h}_v$ represents the static node embedding initialized in Phase 1. This vector $\mathbf{c}_t$ effectively compresses the semantic content of the \textit{active} region of the graph at hop $t$ into a fixed-size representation, which is then passed to the Recurrent Query Updater to determine the next reasoning step. In our example, $\mathbf{c}_1 \approx \text{embedding(``Christopher Nolan'')}$.

% \subsubsection{Step 2.3: Multi-Head Context Aggregation (The "Read" Operation)}
% A standard approach to summarize the reasoning state is to compute a probability-weighted mean of node embeddings: $\mathbf{c}_t = \sum \pi_v \mathbf{h}_v$. However, this operation is destructive when the distribution $\boldsymbol{\pi}^{(t)}$ is multi-modal. For example, if the model identifies two valid but distinct entities (e.g., \textit{DiCaprio} and \textit{Hardy}), averaging their embeddings produces a centroid vector that may lack semantic meaning ("Oversmoothing").

% To preserve distinct semantic signals, we employ a \textbf{Multi-Head Aggregation} strategy. We project the node embeddings into $K$ diverse semantic subspaces and perform weighted aggregation independently in each:
% \begin{equation}
%     \mathbf{c}_t^{(k)} = \sum_{v \in \mathcal{V}_{sub}} \pi_v^{(t)} \cdot \left( \mathbf{W}_{head}^{(k)} \mathbf{h}_v \right), \quad k \in \{1, \dots, K\}
% \end{equation}
% The final context vector $\mathbf{C}_t$ is formed by the simple concatenation of these heads, preserving the independence of the feature subspaces:
% \begin{equation}
%     \mathbf{C}_t = [ \mathbf{c}_t^{(1)} \ ; \ \dots \ ; \ \mathbf{c}_t^{(K)} ]
% \end{equation}
% This allows the RSF-GLLM to maintain a "superposition" of multiple valid reasoning paths, avoiding the information loss inherent in mean-pooling or projection-based fusion.

\subsubsection{Step 2.4: Differentiable Query Update (The \textit{Write} Operation)} \label{sec:query_update}
To enable multi-hop reasoning, the model must handle \textit{intent shifting}: resolving one constraint (e.g., finding the ``Director'') and updating the query to focus on the next (e.g., finding ``Awards'') \cite{hudson2019learning}. We model this evolution using a Gated Recurrent Unit (GRU) that acts as a differentiable instruction pointer.

At each hop $t$, the query vector $\mathbf{q}_t$ is updated based on the retrieved context $\mathbf{c}_t$:
\begin{equation}
    \mathbf{q}_t = \text{GRU}(\mathbf{c}_t, \mathbf{q}_{t-1})
\end{equation}
where $\mathbf{c}_t$ serves as the input (observation) and $\mathbf{q}_{t-1}$ as the hidden state (intent).
Mathematically, the GRU's reset gate $\mathbf{r}_t$ detects features in $\mathbf{c}_t$ that satisfy current query constraints (e.g., identifying ``Christopher Nolan''), while the update gate $\mathbf{z}_t$ rotates the vector in semantic space to point towards the next logical hop (``Awards''). This effectively performs \textit{differentiable instruction subtraction}, removing satisfied constraints from the query representation.

We employ a GRU rather than a Transformer or LSTM for this module due to \textit{sample efficiency} and \textit{inductive bias}. Reasoning chains in KGQA are typically short ($T \le 4$). Transformers often require large-scale data to learn positional dependencies that GRUs capture natively in sequential recurrence. Furthermore, the gating mechanism of the GRU provides a robust \textit{forgetting} bias, which is structurally aligned with the task of discarding resolved query parts, preventing information overload in later hops.

\subsection{Phase 3: Path Extraction \& Textualization}
A critical design choice in RSF-GLLM is the decoupling of graph reasoning from answer generation.
% Unlike methods that feed soft embeddings directly into the LLM, we explicitly extract and textualize the reasoning paths identified by the RSF module.
It ensures that the LLM's high-variance autoregressive gradients do not backpropagate into the GNN, allowing the RSF module to learn stable structural priors independently of the generator's linguistic biases.

% \textbf{Path Decoding Mechanism:}
% Once the RSF module is trained, we extract the most probable reasoning chains from the flow history $\{\boldsymbol{\pi}^{(1)}, \dots, \boldsymbol{\pi}^{(T)}\}$. Since $\boldsymbol{\pi}^{(t)}$ represents the marginal probability of visiting nodes at hop $t$, we recover the joint path probabilities using Beam Search.

% We define a path score $S(P)$ for a sequence $P = (v_0, v_1, \dots, v_T)$ as the cumulative product of node relevance scores:
% \begin{equation}
%     S(P) = \sum_{t=1}^T \log \pi_{v_t}^{(t)}
% \end{equation}
% Starting from the topic entity $v_{topic}$ (where $\pi^{(0)}=1$), we expand the top-$K$ candidates at each hop, constraining transitions to valid graph edges. This yields a set of top-$K$ paths $\mathcal{P}_{top} = \{P_1, \dots, P_K\}$ that "explain" the model's final decision.

\textbf{Path Decoding Mechanism:}
We employ a \textit{Greedy Backtracking} strategy to ensure that extracted paths strictly explain the model's most confident final predictions.
We first identify the top-$K$ answer candidates $\{v_T^{(1)}, \dots, v_T^{(K)}\}$ corresponding to the $K$ highest probability mass values in the final distribution $\boldsymbol{\pi}^{(T)}$.

For each candidate target $v_T^{(k)}$, we reconstruct its causal reasoning chain $P_k = (v_0, \dots, v_T^{(k)})$ in reverse, tracing the flow from the answer back to the source. Specifically, for a node $v_t$ at hop $t$, we select its predecessor $v_{t-1}$ by choosing the neighbor in the previous distribution $\boldsymbol{\pi}^{(t-1)}$ with the highest probability:
\begin{equation}
    v_{t-1} = \underset{u \in \mathcal{N}_{in}(v_t)}{\text{argmax}} \ \pi_u^{(t-1)}
\end{equation}
We repeat this process recursively from $t=T$ down to $t=1$ until we reach the topic entity $v_{topic}$ (where $\pi^{(0)}=1$). This yields a set of top-$K$ paths $\mathcal{P}_{top}$ that represent the most likely structural trajectories leading to the predicted answers.

\textbf{Textualization:} We convert each structured path $P_k$ into a natural language string $S_k$ using a template-based linearizer. For example:
\textit{``Inception (film) $\xrightarrow{directed\_by}$ Christopher Nolan (director) $\xrightarrow{won\_award}$ Academy Award.''}

This two-stage approach prevents the high-variance gradients of the LLM from destabilizing the structural learning of the GNN, ensuring robust optimization for both modules. By standardizing the interface as text, the RSF module becomes a universal plugin compatible with any LLM, including proprietary black-box models (e.g., GPT-4) where gradient access is unavailable. Furthermore, explicit textual paths allow for human verification of the reasoning logic prior to generation, eliminating hallucinated retrieval while significantly reducing computational costs compared to end-to-end inference.

\subsection{Phase 4: LLM Fine-Tuning}
We utilize a pre-trained Large Language Model (LLM) to synthesize the final answer. The LLM is fine-tuned to condition its generation on the extracted paths.

\textbf{Input Construction:} We concatenate the original question $Q$ with the textualized reasoning paths $S_{1:K}$:
\begin{align*}
    \mathbf{X}_{input} &= \texttt{Context: } [S_1; \dots; S_K] \quad \texttt{Question:} Q \\
    \mathbf{Y}_{output} &= \texttt{Answer: } A
\end{align*}

\textbf{Fine-Tuning:} The LLM is trained to generate the ground truth answer $A$ autoregressively, minimizing the negative log-likelihood of the answer tokens.

\subsection{Training Protocol}
We adopt a two-stage training strategy. This decoupling prevents the \textit{gradient noise} from the LLM from destabilizing the graph reasoning process and allows for lighter, more efficient optimization.

\subsubsection{Stage 1: RSF Module Training}
In this stage, the LLM is frozen (or not present). We train the GNN encoders, Attention weights, and Query GRU to maximize the probability of the correct answer node at the final hop $T$.

\textbf{Task Loss ($\mathcal{L}_{task}$):} Let $v_{ans}$ be the ground truth answer node. We construct a target one-hot distribution $\mathbf{y}_{gt} \in \{0,1\}^{|V_{sub}|}$ where $\mathbf{y}_{gt}[v_{ans}] = 1$. We minimize the KL Divergence between the predicted final flow $\boldsymbol{\pi}^{(T)}$ and the ground truth:
\begin{equation}
    \mathcal{L}_{task} = \text{KL}(\mathbf{y}_{gt} \ || \ \boldsymbol{\pi}^{(T)}) = - \sum_{v \in \mathcal{V}_{sub}} \mathbf{y}_{gt}[v] \log \boldsymbol{\pi}^{(T)}_v
\end{equation}
This is equivalent to the Cross-Entropy loss on the final node distribution.

\textbf{Flow Sparsity Regularization ($\mathcal{L}_{flow}$):} A common failure mode in differentiable graph traversal is \textit{flooding}, where the model minimizes risk by assigning uniform non-zero probability to the entire subgraph. This dilutes the context vector $\mathbf{c}_t$ with irrelevant noise.
To enforce focused reasoning, we minimize the entropy of the node relevance distribution $\boldsymbol{\pi}^{(t)}$ at each hop $t$:
\begin{equation}
    \mathcal{L}_{flow} = \frac{1}{T} \sum_{t=1}^T \left( - \sum_{v \in \mathcal{V}_{sub}} \pi_v^{(t)} \log (\pi_v^{(t)} + \epsilon) \right)
\end{equation}
where $\epsilon$ is a small constant for numerical stability.
Minimizing entropy encourages the distribution to be \textit{spiky}, forcing the model to select a distinct reasoning path (e.g., exclusively highlighting ``Christopher Nolan'') rather than smearing probability mass across all neighbors. This mimics discrete logical steps while retaining differentiability.

\textbf{Total Stage 1 Loss:} $\mathcal{L}_{Stage1} = \mathcal{L}_{task} + \lambda_1 \mathcal{L}_{flow}$.

\subsubsection{Stage 2: LLM Optimization}
After Stage 1 converges, we freeze the RSF module. We run the RSF module on the training set to extract the top-$K$ paths and generate the prompt $\mathbf{X}_{input}$. We then fine-tune the LLM using standard Causal Language Modeling (CLM) loss:
\begin{equation}
    \mathcal{L}_{Stage2} = - \sum_{j=1}^{|A|} \log P_{\theta}(a_j \mid a_{<j}, \mathbf{X}_{input})
\end{equation}

\subsection{Theoretical Analysis}
We analyze three critical theoretical properties of RSF-GLLM: its ability to converge to crisp reasoning paths (Sparsity), its capacity to traverse semantically dissimilar nodes (Expressivity), and its capability to ensure structurally faithful reasoning. Detailed proofs of all the theorems are provided in Appendix \ref{sec:proofs}.

\begin{theorem}[Sparsity Convergence]
\label{thm:sparsity}
The global minima of the flow sparsity loss $\mathcal{L}_{flow}(\boldsymbol{\pi}) = - \sum \pi_i \log \pi_i$ lie exclusively at the vertices of the probability simplex $\Delta^{N-1}$.
\end{theorem}

% \textit{Proof Sketch.} The negative entropy function is strictly convex. A fundamental result in convex optimization states that maximizing a strictly convex function (or minimizing our concave entropy loss) over a convex set (the probability simplex) forces the solution to the extreme points (vertices). See Appendix \ref{app:proof_sparsity} for the rigorous derivation.

The above property ensures that the model avoids flooding the graph with uniform probability, which would dilute the reasoning signal with noise. Instead, the optimization landscape mathematically incentivizes the model to make decisive, unambiguous choices at each hop, mimicking discrete logical steps.

\begin{theorem}[Semantic Gap Bridging]
\label{thm:expressivity}
Let $P = (v_0 \xrightarrow{r_1} v_1 \xrightarrow{r_2} \cdots \xrightarrow{r_K} v_K)$ be a valid reasoning path satisfying: \emph{(i)} intermediate nodes possess zero semantic similarity to the question (i.e., $\mathbf{h}_{v_t}^\top \mathbf{q}_0 \approx 0$ for $0 < t < K$), and \emph{(ii)} for each hop $t \in \{1, \ldots, K\}$, $v_t$ is the unique $r_t$-neighbor of $v_{t-1}$ in $\mathcal{G}_{sub}$ (no fan-out). There exists a parameter configuration for the Recurrent Query Updater and Relation Attention such that $\pi_{v_t}^{(t)} \approx 1$ at each hop $t \in \{1, \ldots, K\}$, solely via structural propagation.
\end{theorem}

% \textit{Proof Sketch.} We construct a solution by induction. We show that the GRU functions as a rotation operator in the semantic space. By aligning the query vector $\mathbf{q}_t$ with the relation embedding $\mathbf{e}_{r_{t+1}}$ at each step, the model maximizes the structural flow term $\Phi_v^{(t)}$, rendering the low content score $\rho_v^{(t)}$ irrelevant to the final selection. See Appendix \ref{app:proof_expressivity} for the full construction.

This theorem proves that RSF-GLLM can traverse structural bridges (like \textit{Christopher Nolan}) purely based on relation semantics. This allows it to solve complex queries where the intermediate evidence has no lexical overlap with the original question, effectively bypassing the \textit{semantic gap}.

\begin{theorem}[Guaranteed Structural Faithfulness]
\label{thm:consistency}
Let $v_T$ be a target node at hop $T$ with non-zero structural flow, i.e., $\Phi_{v_T}^{(T)} > 0$. The greedy backtracking procedure, defined by $v_{t-1} = \operatorname{argmax}_{u \in \mathcal{N}_{in}(v_t)} \pi_u^{(t-1)}$, is guaranteed to reconstruct a continuous, valid path $P = (v_0, \dots, v_T)$ connecting the source entity $v_{topic}$ to $v_T$ within $T$ steps.
\end{theorem}

This theorem provides a critical guaranty for factual grounding. It ensures that every answer generated by our pipeline is supported by an explicit, verifiable reasoning chain existing in the knowledge graph. Unlike heuristic search methods that may fail to connect endpoints or LLM based approaches with the risk of hallucinations, \textbf{our flow-based backtracking mathematically ensures connectivity, thereby enforcing structural faithfulness by design}.

\section{Experiments}
\label{sec:experiments}

We evaluate RSF-GLLM on two benchmarks requiring multi-hop reasoning over Knowledge Graphs. Our experiments are designed to verify: (1) The superiority of the decoupled \textit{Graph-to-LLM} approach over rigid retrievers, hallucination-prone LLMs and extremely expensive agentic approaches; (2) The efficiency gains from our lightweight RSF module; and (3) The empirical validity of our theoretical claims and model ablations.

\subsection{Experimental Setup}

\textbf{Datasets.} We utilize two standard KGQA benchmarks:

\textbf{WebQSP} \cite{yih2016value}: Contains approx. 4,700 queries derived from Google Search. It requires up to 2 hops of reasoning and is characterized by high semantic variety.

\textbf{CWQ (Complex WebQuestions)} \cite{talmor2018web}: A dataset focusing on complex, compositional queries requiring up to 4 hops. It serves as a stress test for the Recurrent Query Updater's ability to maintain state over long chains.

\textbf{Baselines:} All baselines are listed in Table \ref{tab:main_results} and discussed in Appendix \ref{sec:baslines}. We have only selected baselines that are using open source LLMs with a similar number of model parameters as in Qwen3-8B to have a fair comparison. The reported numbers are mostly taken from the source papers.

\textbf{Implementation Details.}
RSF-GLLM is implemented in PyTorch using PyTorch Geometric. Following RoG~\cite{luo2024reasoning}, we use preprocessed WebQSP and CWQ datasets with entities pre-linked to Freebase. We extract $K$-hop subgraphs via BFS ($K=2$ for WebQSP, $K=4$ for CWQ). Node and relation embeddings are obtained from Qwen3-Embedding.. The RSF module is trained for up to 10 epochs using AdamW (lr=$5\times10^{-5}$, weight decay 0.01, $\lambda_1=0.1$). For the LLM reader, we use Qwen3-8B with extracted reasoning paths. We report Hit@1 and F1 following standard evaluation metrics. All experiments use a single NVIDIA A100 GPU. Baseline details are in Appendix~\ref{sec:baslines}.

\begin{table*}[t!]
    \centering
    \caption{Performance comparison on WebQSP and CWQ benchmarks. Best results are in \textbf{bold}, second-best are \underline{underlined}. The ``Agentic Search'' category (shaded) relies on computationally expensive iterative LLM calls, serving as an oracle-style upper bound. RSF-GLLM uses Qwen3-8B or LLaMa-2-7B as the LLM answer generator.}
    \label{tab:main_results}
    \resizebox{0.85\textwidth}{!}{%
    \begin{tabular}{ll|cc|cc}
        \toprule
        & & \multicolumn{2}{c|}{\textbf{WebQSP}} & \multicolumn{2}{c}{\textbf{CWQ}} \\
        \textbf{Category} & \textbf{Method} & Hit@1 & F1 & Hit@1 & F1 \\
        \midrule
        \multirow{4}{*}{\textit{Embedding-based}} 
        & KV-Mem \cite{miller-etal-2016-key} & 46.7 & 34.5 & 18.4 & 15.7 \\
        & EmbedKGQA \cite{saxena-etal-2020-improving} & 66.6 & -- & 45.9 & -- \\
        & NSM \cite{he2021improving} & 68.7 & 62.8 & 47.6 & 42.4 \\
        & TransferNet \cite{shi-etal-2021-transfernet} & 71.4 & -- & 48.6 & -- \\
        \midrule
        \multirow{5}{*}{\textit{Graph Retrieval}} 
        & GraftNet \cite{sun2018open} & 66.4 & 60.4 & 36.8 & 32.7 \\
        & PullNet \cite{sun2019pullnet} & 68.1 & -- & 45.9 & -- \\
        & SR+NSM \cite{zhang-etal-2022-subgraph} & 68.9 & 64.1 & 50.2 & 47.1 \\
        & ReaRev \cite{mavromatis-karypis-2022-rearev} & 76.4 & 70.9 & 52.9 & 47.8 \\
        & UniKGQA \cite{jiang2023unikgqaunifiedretrievalreasoning} & 77.2 & 72.2 & 51.2 & 49.1 \\
        \midrule
        \multirow{2}{*}{\textit{LLM Reasoning}} 
        & Qwen3-8B (Zero-shot) & 50.1 & 34.0 & 27.5 & 25.8 \\
        & LLaMA3.1-8B (Zero-shot) & 55.1 & 35.6 & 27.7 & 22.8 \\
        \midrule
        \multirow{6}{*}{\textit{LLM + KG}} 
        & KD-CoT \cite{wang2023knowledgedrivencotexploringfaithful} & 68.6 & 52.5 & 55.7 & -- \\
        & DECAF (FiD-3B) \cite{yu2023decaf} & 82.1 & 78.8 & 70.4 & -- \\
        & RoG (LLaMA2-7B) \cite{luo2024reasoning} & 85.7 & 70.8 & 62.6 & 56.2 \\
        & G-Retreiver \cite{he2024gretrieverretrievalaugmentedgenerationtextual}& 70.1 & -- & -- & -- \\
        & GNN-RAG (LLaMA2-7B) \cite{mavromatis-karypis-2025-gnn} & 80.6 & 71.3 & 61.7 & 59.4 \\
        & GNN-RAG + RA (LLaMA2-7B) \cite{mavromatis-karypis-2025-gnn} & 82.8 & 73.5 & 62.8 & 60.4 \\
        \midrule
        \midrule
        \multirow{2}{*}{\textit{\textbf{Ours}}} & \textbf{RSF-GLLM (LLaMA2-7B)} & \underline{89.50} & \underline{78.93} & 66.44 & \underline{60.80} \\
        & \textbf{RSF-GLLM (Qwen3-8B)} & \textbf{90.45} & \textbf{79.15} & \textbf{67.39} & \textbf{61.87} \\
        \midrule
        \midrule
        % AGENTIC SEARCH (Shaded Background)
        \rowcolor{gray!15} 
        & EffiQA (Llama3.1-8B) \cite{dong-etal-2025-effiqa} & 82.9 & -- & 69.5 & -- \\
        \rowcolor{gray!15} \multirow{-2}{*}{\textit{Agentic Search}} 
        & FD-PORT (LLaMA3-8B) \cite{10.1145/3726302.3729980} & 89.2 & -- & 74.5 & -- \\
        \bottomrule
    \end{tabular}
    }
\end{table*}

\subsection{Main Performance Results}
\label{sec:results_analysis}

Table \ref{tab:main_results} presents the performance comparison on WebQSP and CWQ. Our proposed RSF-GLLM achieves state-of-the-art results in Hit@1 on WebQSP and competitive performance on the complex CWQ benchmark, validating the effectiveness of decoupling differentiable graph reasoning from LLM generation.

\textbf{Comparison with State-of-the-Art LLM+KG Methods:}
RSF-GLLM demonstrates robust superiority in answer accuracy (Hit@1) over leading baselines in the \textit{LLM + KG} category. On WebQSP, we achieve \textbf{90.45\%} Hit@1, significantly outperforming \textit{RoG} (85.7\%), \textit{GNN-RAG} (80.6\%), and \textit{DECAF} (82.1\%).
RSF-GLLM also surpasses \textit{DECAF} on F1 (79.15\% vs.\ 78.8\%), demonstrating that our framework achieves superior performance on both Hit@1 and F1 metrics simultaneously.

Crucially, our performance is achieved with superior efficiency. While methods like \textit{GNN-RAG} and \textit{DECAF} require heavy retrieval-augmented generation pipelines or multiple LLM calls, RSF-GLLM employs a lightweight, differentiable RSF module for path extraction. This allows us to use a \textbf{single LLM call} for the final answer generation, drastically reducing computational overhead while maintaining high structural grounding. For example, the average inference time per question by RSF-GLLM was around $0.25$ sec on both data sets on an A100 80GB machine, which is significantly faster than other competitive baselines.

\textbf{Comparison with Agentic Search (Oracle Baselines):}
The \textit{Agentic Search} category (shaded in Table \ref{tab:main_results}) treats KG traversal as a sequential decision process involving iterative LLM invocations (e.g., FD-PORT requires 50+ calls). We consider these methods as high-cost \textit{oracles}. Remarkably, RSF-GLLM surpasses these computationally expensive frameworks on WebQSP (90.45\% vs. 89.2\% for FD-PORT). This indicates that our differentiable Recurrent Soft-Flow mechanism effectively captures complex reasoning structures akin to agentic planning, but at a fraction of the latency and cost.

\textbf{Note on the WebQSP--CWQ Hit@1 gap.}
The raw Hit@1 gap between WebQSP (90.45\%) and CWQ (67.39\%) partially reflects a structural property of the datasets rather than a reasoning limitation. Under the standard $K$-hop BFS subgraph construction adopted by all retrieval-based baselines, 19.3\% of CWQ test questions have their gold answer entity \textit{outside} the extracted subgraph, imposing a theoretical Hit@1 ceiling of approximately $80.7\%$ on CWQ. WebQSP's coverage under $K{=}2$ is $95.7\%$, yielding a naturally higher ceiling. A substantial portion of the observed cross-dataset Hit@1 gap is therefore attributable to subgraph coverage limits, not model failure; the $K$-hop bottleneck this exposes motivates the Dynamic Beam Expansion extension discussed in Section~\ref{sec:conclusion}.

\begin{table*}[t]
    \centering
    \caption{Comparison of retrieval efficiency and performance. Efficiency columns (Params, Time, Mem) are measured for 1 epoch training on CWQ.
    \textbf{\#LLM} denotes the number of LLM calls required strictly during the path retrieval phase.
    H@$k$ measures whether the ground truth answer appears in the top-$k$ candidates provided to the LLM reader.}
    \label{tab:retrieval_efficiency}
    \resizebox{0.8\textwidth}{!}{%
    \begin{tabular}{l|cccc|ccc|ccc}
        \toprule
        & & & & & \multicolumn{3}{c|}{\textbf{WebQSP}} & \multicolumn{3}{c}{\textbf{CWQ}} \\
        \textbf{Method} & \textbf{Params} & \textbf{Time (s)} & \textbf{Mem (MiB)} & \textbf{\#LLM} & \textbf{H@1} & \textbf{H@5} & \textbf{H@10} & \textbf{H@1} & \textbf{H@5} & \textbf{H@10} \\
        \midrule
        RoG & 6.7B & 9671 & 64,568 & 1 & 59.3 & 73.5 & 77.7 & 26.1 & 46.4 & 54.1 \\
        GNN-RAG & \textbf{2.9M} & \textbf{737} & \textbf{156} & \textbf{0} & \underline{76.2} & 85.8 & 88.0 & \underline{52.9} & 64.2 & 67.0 \\
        GNN-RAG+RA & 6.7B & 10406 & 64,579 & 1 & \textbf{76.3} & \underline{87.6} & \underline{91.0} & \textbf{55.1} & \textbf{69.3} & \underline{72.1} \\
        \midrule
        \textbf{RSF (Ours)} & \underline{176M} & \underline{993} & \underline{2,770} & \textbf{0} & 73.9 & \textbf{89.4} & \textbf{93.3} & 51.3 & \underline{67.7} & \textbf{73.1} \\
        \bottomrule
    \end{tabular}
    }
\end{table*}

\subsection{Retrieval Performance and Efficiency}
\label{sec:retrieval_analysis}

Table \ref{tab:retrieval_efficiency} presents a comprehensive evaluation of retrieval quality versus computational cost. The results highlight distinct limitations in existing baselines that RSF-GLLM effectively overcomes.

\textbf{Limitations of LLM-Heavy Retrievers:}
Methodologies like \textit{RoG} rely on fine-tuning large language models (6.7B parameters) to generate reasoning paths directly. Despite its massive size and computational demand ($>$64GB memory, $\sim$10k seconds/epoch), RoG performs significantly worse than our approach, achieving only 54.1\% Hit@10 compared to RSF's 72.1\%. This indicates that generative LLMs, while fluent, often hallucinate relations or lose track of structural constraints in purely parametric generation.
Similarly, \textit{GNN-RAG+RA}, an agentic extension of GNN-RAG that integrates an LLM reasoner, improves retrieval to 71.1\%. However, this performance gain comes at the cost of extreme complexity—matching RoG's prohibitive resource footprint and requiring slow, iterative LLM calls during retrieval. Even with this agentic overhead, it fails to match the retrieval precision of our lightweight module.

\textbf{Limitations of Pure GNNs:}
At the other end of the spectrum, the standard \textit{GNN-RAG} retriever is extremely lightweight (2.9M parameters) and fast. However, its performance (67.0\% Hit@10) is inadequate for complex multi-hop QA. Relying primarily on static embedding similarity, it lacks the capacity to model the dynamic, state-dependent reasoning required for compositional queries.

\textbf{The RSF Advantage:}
RSF-GLLM occupies a unique \textit{sweet spot}. By designing a specialized differentiable reasoner, we reduce the model size to \textbf{176M parameters} ($38\times$ smaller than RoG/GNN-RAG+RA) and eliminate the need for LLM calls during retrieval. Unlike the static GNN-RAG, our Recurrent Soft-Flow mechanism is expressive enough to model complex agentic-like planning steps (Section \ref{sec:qualitative}), allowing it to achieve state-of-the-art retrieval performance (\textbf{72.1\% Hit@10}). This demonstrates that a structure-aware module can outperform massive general-purpose LLMs in graph navigation tasks while remaining efficient enough for training on consumer-grade hardware.

\subsection{Ablation Study}

We validate the contribution of each RSF component through systematic ablations on WebQSP (Table~\ref{tab:ablation}).

\begin{table}[t]
    \centering
    \caption{Ablation study of RSF module on WebQSP. $\Delta$ indicates change from the unablated RSF module performance on the Hit@10 metric.}
    \label{tab:ablation}
    \small
    \begin{tabular}{l|ccc}
        \toprule
        \textbf{Configuration} & \textbf{Hit@10} & \textbf{F1} & \textbf{$\Delta$} \\
        \midrule
        \textbf{Full RSF} & \textbf{89.2} & \textbf{52.4} & --- \\
        \midrule
        w/o Query Update & 83.6 & 46.2 & $-5.6$ \\
        w/o Content Bias & 82.8 & 45.1 & $-6.4$ \\
        w/o Dynamic Gate & 86.8 & 49.8 & $-2.4$ \\
        w/o Flow Sparsity & 85.6 & 47.5 & $-3.6$ \\
        \bottomrule
    \end{tabular}
\end{table}

\begin{table}[t]
    \centering
    \caption{Effect of RSF paths vs.\ QA-pair memorization. ``No RSF'' fine-tunes the LLM on (question, answer) pairs only, without reasoning paths.}
    \label{tab:no_rsf_ablation}
    \small
    \resizebox{\columnwidth}{!}{%
    \begin{tabular}{l|cc|cc}
        \toprule
        & \multicolumn{2}{c|}{\textbf{WebQSP}} & \multicolumn{2}{c}{\textbf{CWQ}} \\
        \textbf{Condition} & H@1 & F1 & H@1 & F1 \\
        \midrule
        LLaMA-2-7B (no RSF)  & 73.16 & 55.65 & 46.42 & 40.34 \\
        Qwen3-8B (no RSF)    & 71.94 & 54.27 & 48.21 & 43.19 \\
        \midrule
        LLaMA-2-7B (w/ RSF)  & \textbf{89.50} & \textbf{78.93} & \textbf{66.44} & \textbf{60.80} \\
        $\Delta$ (LLaMA)     & +16.34 & +23.28 & +20.02 & +20.46 \\
        \bottomrule
    \end{tabular}}
\end{table}

\begin{table}[t]
    \centering
    \caption{Sensitivity to subgraph radius $K$ on CWQ (RSF retrieval Hit@1).}
    \label{tab:khop_sensitivity}
    \small
    \begin{tabular}{c|cccc}
        \toprule
        \textbf{$K$-hop} & 1 & 2 & 3 & 4 \\
        \midrule
        \textbf{Hit@1} & 44.35 & 45.77 & 48.34 & \textbf{51.30} \\
        \bottomrule
    \end{tabular}
\end{table}

\textbf{Content Bias} ($-6.4\%$) is the most critical component. Without the content score $\rho_v$, the model cannot disambiguate fan-out scenarios where multiple neighbors share the same valid relation. When querying about a specific movie by a prolific director, structural flow $\Phi_v$ distributes uniformly across all films, making the correct target indistinguishable.

\textbf{Query Update} ($-5.6\%$) disables the GRU-based intent evolution. Without it, the query vector $\mathbf{q}$ remains frozen across hops, preventing ``differentiable instruction subtraction''---the model cannot shift attention from resolving one constraint (e.g., ``find the director'') to the next (e.g., ``find awards'').

\textbf{Dynamic Gate} ($-2.4\%$) validates the Semantic Gap hypothesis. Fixing $\lambda \equiv 1$ (see Eq.~\ref{eq:gatingLambda}) forces content matching at every hop, suppressing valid bridge entities like ``Christopher Nolan'' that share no lexical overlap with query terms like ``awards.''

\textbf{Flow Sparsity by setting $\lambda_1 = 0$ in Total Stage 1 Loss} ($-3.6\%$) prevents ``flooding'' where probability spreads uniformly, degrading the GRU's ability to track reasoning progress.

\textbf{Effect of $\lambda_1$} Figure~\ref{fig:flow_sparsity} reveals an inverted U-curve analogous to \textit{temperature in softmax}: $\lambda_1 \to 0$ yields diffuse distributions ($\sim$11 candidates), while $\lambda_1 \to 0.5$ causes premature commitment ($\sim$2 candidates) where early-hop errors propagate irrecoverably. The optimal $\lambda_1=0.1$ achieves $\sim$4 effective candidates, validating Theorem~\ref{thm:sparsity}.

\begin{figure}[t]
    \centering
    \includegraphics[width=0.75\columnwidth]{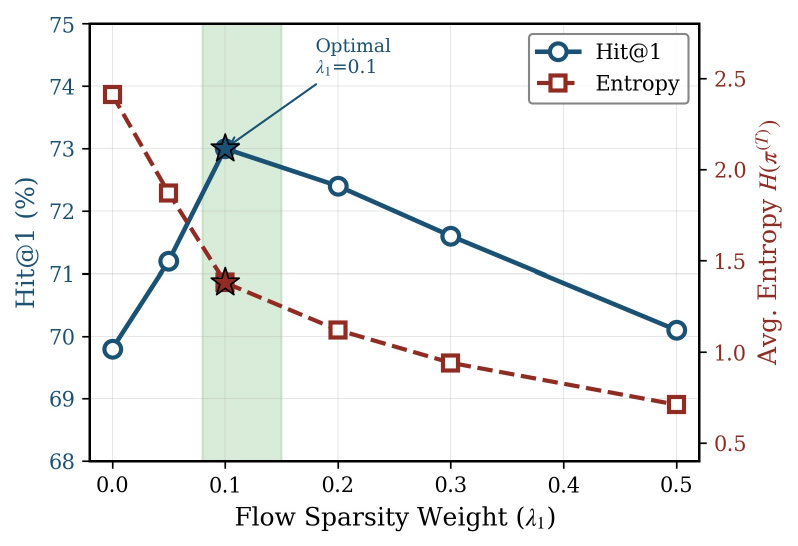}
    \caption{Effect of $\lambda_1$ on Hit@1 and entropy. Optimal at $\lambda_1=0.1$.}
    \label{fig:flow_sparsity}
\end{figure}

More experiments on dynamic gating (including a CWQ-specific gate ablation in Table~\ref{tab:cwq_gate_ablation}) and case studies are presented in Appendix \ref{sec:lambda_analysis} and \ref{sec:qualitative} respectively.

\textbf{Isolating RSF Path Contribution:}
To verify that performance gains stem from RSF-extracted reasoning paths rather than QA-pair memorization, we trained an identical LLM on (question, answer) pairs without any RSF paths (Table~\ref{tab:no_rsf_ablation}). The $+16$--$23$pt gap across all metrics confirms that the LLM cannot recover structural and factual information from questions alone; RSF paths provide essential grounding that drives the performance improvement.

\textbf{Sensitivity to Subgraph Radius:}
We conducted a sensitivity study varying the $K$-hop subgraph radius on CWQ (Table~\ref{tab:khop_sensitivity}). Results show a consistent improvement in RSF retrieval Hit@1 with increased hop budget (44.35\% at $K{=}1$ to 51.30\% at $K{=}4$), highlighting the importance of multi-hop reasoning capacity and our design choice to include self-loops.

\section{Conclusion and Future Work}
\label{sec:conclusion}

We introduced RSF-GLLM, a framework that decouples differentiable graph reasoning from LLM generation to bridge the semantic gap in KGQA. Our Recurrent Soft-Flow module employs dynamic gating and flow sparsity regularization to extract interpretable reasoning paths, avoiding the computational cost of iterative agentic loops. Experiments demonstrate that RSF-GLLM achieves state-of-the-art retrieval accuracy and a $21\times$ inference speedup over baselines by requiring only a single LLM call.

Future work will address two limitations of the current pipeline through extensions naturally enabled by the decoupled RSF--LLM design. To scale beyond the pre-extracted $K$-hop subgraph regime---which can bottleneck recall on massive real-world KGs---we plan to introduce \textit{Dynamic Beam Expansion}, a reasoning-step-driven traversal in which the model dynamically queries the backend graph database for the 1-hop neighbors of \textit{active} nodes ($\pi^{(t)}_v > \theta$), guided by the Flow Sparsity regularizer to keep the active set small and by an adaptive GRU-based halting criterion that eliminates the fixed hop count $T$. To mitigate the popularity bias of the LLM generator (Appendix E, Failure Mode 2), we plan to integrate \textit{Context-Aware Contrastive Decoding} at inference: by contrasting LLM predictions conditioned on the question alone (capturing the parametric prior) against predictions conditioned on the question together with the extracted reasoning paths, we can dynamically penalize tokens favored purely by the internal prior and amplify predictions supported by retrieved structural context. Because RSF cleanly decouples reasoning from generation, this contrastive correction can be applied without retraining the underlying LLM. Beyond these, we plan to extend RSF-GLLM to temporal and multi-modal knowledge graphs and to investigate joint RSF--LLM optimization.

\section*{Impact Statement}

This paper presents work whose goal is to advance the field of Machine Learning, specifically in interpretable and efficient Question Answering. Our proposed RSF-GLLM framework explicitly addresses the critical issue of hallucination in Large Language Models by grounding generation in verifiable Knowledge Graph paths, thereby contributing to more reliable and trustworthy AI systems. Furthermore, by decoupling reasoning from generation, our approach significantly reduces the computational resources required for multi-hop inference compared to iterative agentic baselines, promoting more energy-efficient and environmentally sustainable AI deployment.


\bibliography{example_paper}
\bibliographystyle{icml2026}

%%%%%%%%%%%%%%%%%%%%%%%%%%%%%%%%%%%%%%%%%%%%%%%%%%%%%%%%%%%%%%%%%%%%%%%%%%%%%%%
%%%%%%%%%%%%%%%%%%%%%%%%%%%%%%%%%%%%%%%%%%%%%%%%%%%%%%%%%%%%%%%%%%%%%%%%%%%%%%%
% APPENDIX
%%%%%%%%%%%%%%%%%%%%%%%%%%%%%%%%%%%%%%%%%%%%%%%%%%%%%%%%%%%%%%%%%%%%%%%%%%%%%%%
%%%%%%%%%%%%%%%%%%%%%%%%%%%%%%%%%%%%%%%%%%%%%%%%%%%%%%%%%%%%%%%%%%%%%%%%%%%%%%%
\newpage
\appendix
\onecolumn

\section*{Appendix}


% You can have as much text here as you want. The main body must be at most $8$
% pages long. For the final version, one more page can be added. If you want, you
% can use an appendix like this one.

% The $\mathtt{\backslash onecolumn}$ command above can be kept in place if you
% prefer a one-column appendix, or can be removed if you prefer a two-column
% appendix.  Apart from this possible change, the style (font size, spacing,
% margins, page numbering, etc.) should be kept the same as the main body.

\section{Theoretical Proofs}\label{sec:proofs}

\subsection{Proof of Theorem \ref{thm:sparsity} (Sparsity Convergence)}
\label{app:proof_sparsity}

\textbf{Statement:} \textit{The global minima of the flow sparsity loss $\mathcal{L}_{flow}(\boldsymbol{\pi})$ lie exclusively at the vertices of the simplex $\Delta^{N-1}$.}

\begin{proof}
Let $\boldsymbol{\pi} \in \mathbb{R}^N$ be the relevance distribution vector satisfying the simplex constraints: $\sum_{i=1}^N \pi_i = 1$ and $\pi_i \ge 0$.
The Flow Sparsity Regularization term is defined as the entropy of the distribution:
\begin{equation}
    \mathcal{L}_{flow}(\boldsymbol{\pi}) = H(\boldsymbol{\pi}) = - \sum_{i=1}^N \pi_i \log \pi_i
\end{equation}
To analyze the minima, we inspect the curvature of the function. The Hessian of $\mathcal{L}_{flow}$ with respect to $\boldsymbol{\pi}$ is a diagonal matrix with elements:
\begin{equation}
    \frac{\partial^2 \mathcal{L}_{flow}}{\partial \pi_i^2} = - \frac{1}{\pi_i}
\end{equation}
Since $\pi_i > 0$ in the interior of the simplex, the second derivative is strictly negative. Thus, $\mathcal{L}_{flow}(\boldsymbol{\pi})$ is a \textbf{strictly concave} function over the domain.

A fundamental result in convex optimization states that the minimum of a strictly concave function over a convex polytope (such as the simplex $\Delta^{N-1}$) must occur at one of the extreme points (vertices) of the polytope.
The vertices of the probability simplex are the standard basis vectors $\{\mathbf{e}_1, \dots, \mathbf{e}_N\}$, which correspond to one-hot distributions (i.e., $\pi_k = 1$ for some $k$, and $0$ otherwise).

Evaluating the loss at a vertex $\mathbf{e}_k$:
\begin{equation}
    \mathcal{L}_{flow}(\mathbf{e}_k) = - (1 \log 1 + \sum_{j \neq k} 0 \log 0) = 0
\end{equation}
(taking $\lim_{x\to 0} x \log x = 0$).
Since entropy is non-negative $H(\boldsymbol{\pi}) \ge 0$, the value $0$ is the global minimum. Therefore, minimizing $\mathcal{L}_{flow}$ forces the model to select a single crisp path node at each hop, mimicking discrete reasoning.
\end{proof}

\subsection{Proof of Theorem \ref{thm:expressivity} (Semantic Gap Bridging)}
\label{app:proof_expressivity}

\textbf{Statement:} \textit{Let $P = (v_0 \xrightarrow{r_1} v_1 \xrightarrow{r_2} \cdots \xrightarrow{r_K} v_K)$ be a valid reasoning path satisfying: (i) intermediate nodes possess zero semantic similarity to the question (i.e., $\mathbf{h}_{v_t}^\top \mathbf{q}_0 \approx 0$ for $0 < t < K$), and (ii) for each hop $t \in \{1, \ldots, K\}$, $v_t$ is the unique $r_t$-neighbor of $v_{t-1}$ in $\mathcal{G}_{sub}$ (no fan-out). There exists a parameter configuration for the Recurrent Query Updater and Relation Attention such that $\pi_{v_t}^{(t)} \approx 1$ at each hop $t \in \{1, \ldots, K\}$, solely via structural propagation.}

\begin{proof}
We provide a constructive proof by induction on the hop index $t$.
Let the path be $P = (v_0 \xrightarrow{r_1} v_1 \xrightarrow{r_2} \dots \xrightarrow{r_K} v_K)$, and let the Semantic Gap assumption hold: $\text{sim}(\mathbf{h}_{v_t}, \mathbf{q}_0) = 0$ for all $t > 0$.
We aim to show that for each hop $t \in \{1, \ldots, K\}$, there exists a parameter configuration such that $\pi_{v_t}^{(t)} \approx 1$.

\textbf{Step 0: Gate Construction (Suppressing Content Bias).}
Set $b_\lambda \ll 0$ in Eq.~\ref{eq:gatingLambda}. Then $\lambda^{(t)} = \sigma(\mathbf{w}_\lambda^\top \mathbf{q}_{t-1} + b_\lambda) \approx 0$ for all $t$, regardless of the query state. The probabilistic aggregation (Eq.~\ref{eq:probAgg}) then simplifies. Since $\exp(\log(x)) = x$, it reduces directly to a ratio of structural flow values:
\begin{equation}
    \pi_v^{(t)} \approx \frac{\Phi_v^{(t)}}{\displaystyle\sum_z \Phi_z^{(t)}}
\end{equation}
(the additive $\epsilon$ terms cancel and are negligible in the regime $c \to \infty$ established in Step~1).

\textbf{Step 1: Relation Alignment (Existence of $\mathbf{W}_{att}$).}
Assume the relation embeddings $\{\mathbf{e}_r\}_{r \in \mathcal{R}}$ are linearly independent. Their dual basis $\{\mathbf{d}_r\}_{r \in \mathcal{R}}$ then exists, satisfying $\mathbf{d}_r^\top \mathbf{e}_{r'} = \delta_{rr'}$ (Kronecker delta). For a target relation $r_t$, define the direction:
\begin{equation}
    \mathbf{f}_{r_t} = c\!\left(\mathbf{d}_{r_t} - \alpha \sum_{r'' \neq r_t} \mathbf{d}_{r''}\right), \quad c,\, \alpha > 0
\end{equation}
Then $\mathbf{f}_{r_t}^\top \mathbf{e}_{r_t} = c$ and $\mathbf{f}_{r_t}^\top \mathbf{e}_{r''} = -c\alpha$ for all $r'' \neq r_t$.
We fix $\mathbf{W}_{att}$ as any full-rank matrix from the query space to the relation embedding space; its specific value does not affect the construction. Since $\mathbf{W}_{att}$ is full-rank, for each hop's target direction $\mathbf{f}_{r_t}$ there exists a query state satisfying $\mathbf{W}_{att}^\top \mathbf{q}_{t-1} = \mathbf{f}_{r_t}$; Step~2 below constructs GRU weights to produce these target query states at each hop.
With $\mathbf{q}_{t-1}$ so positioned, the relation attention scores (Eq.~\ref{eq:relatt}) evaluate to $\mathbf{q}_{t-1}^\top \mathbf{W}_{att} \mathbf{e}_{r_t} = \mathbf{f}_{r_t}^\top \mathbf{e}_{r_t} = c$ and $\mathbf{q}_{t-1}^\top \mathbf{W}_{att} \mathbf{e}_{r''} = \mathbf{f}_{r_t}^\top \mathbf{e}_{r''} = -c\alpha$, yielding:
\begin{equation}
    \beta_{r_t}^{(t)} = \sigma(c) \;\to\; 1 \qquad \text{and} \qquad \beta_{r''}^{(t)} = \sigma(-c\alpha) \;\to\; 0 \quad \forall\, r'' \neq r_t
\end{equation}
as $c \to \infty$.

With $\beta_{r_t}^{(t)} \approx 1$ and $\beta_{r''}^{(t)} \approx 0$ for all $r'' \neq r_t$, the structural flow propagation concentrates on nodes reachable from $v_{t-1}$ via $r_t$. Using the inductive hypothesis $\pi_{v_{t-1}}^{(t-1)} \approx 1$:
\begin{align}
    \Phi_{v_t}^{(t)} &= \sum_{(u,\,r,\,v_t)\,\in\,\mathcal{N}_{in}(v_t)} \pi_u^{(t-1)} \cdot \beta_r^{(t)} \;\approx\; \pi_{v_{t-1}}^{(t-1)} \cdot \beta_{r_t}^{(t)} \;\approx\; 1 \\
    \Phi_z^{(t)} &\approx\; 0 \qquad \text{for all } z \text{ reachable from } v_{t-1} \text{ only via } r'' \neq r_t
\end{align}
By condition (ii) of the theorem, $v_t$ is the unique $r_t$-neighbor of $v_{t-1}$ in $\mathcal{G}_{sub}$, so no other node receives $\Phi \approx 1$. Substituting into the simplified softmax of Step 0, we obtain $\pi_{v_t}^{(t)} \approx 1$.

\textbf{Step 2: Dynamic Update (Existence of GRU parameters).}
After traversing edge $r_t$, the query must update so that $\mathbf{W}_{att}^\top \mathbf{q}_t = \mathbf{f}_{r_{t+1}}$, aligning with the next target relation $r_{t+1}$. The update rule is $\mathbf{q}_t = \text{GRU}(\mathbf{c}_t, \mathbf{q}_{t-1})$, where $\mathbf{c}_t = \sum_v \pi_v^{(t)} \mathbf{h}_v \approx \mathbf{h}_{v_t}$ since $\pi_{v_t}^{(t)} \approx 1$.
Since GRUs with sufficient hidden state dimension are universal approximators of dynamical systems, there exist weights implementing a state transition $f$ such that:
\begin{equation}
    \mathbf{q}_t = f(\mathbf{q}_{t-1},\, \mathbf{h}_{v_t}) \quad \text{s.t.} \quad \mathbf{W}_{att}^\top \mathbf{q}_t = \mathbf{f}_{r_{t+1}}
\end{equation}
Even when $\mathbf{h}_{v_t}$ shares no lexical similarity with the query, it acts as a distinct structural key: the GRU learns the state transition ``currently aligned with relation $r_t$; upon observing node $v_t$, realign with $r_{t+1}$.''

\textbf{Conclusion.}
The base case $\pi_{v_0}^{(0)} = 1$ holds by initialization. Steps 1--2 establish the inductive step: $\pi_{v_{t-1}}^{(t-1)} \approx 1$ implies $\pi_{v_t}^{(t)} \approx 1$. Therefore, at each hop $t \in \{1, \ldots, K\}$, the flow is concentrated on the correct path node $v_t$, confirming that the RSF architecture bridges the semantic gap using only structural cues (relation type attention $\beta_r^{(t)}$, Eq.~\ref{eq:relatt}) and state memory (GRU), without relying on node-question lexical similarity.
\end{proof}

\textbf{Remark on the linear independence assumption.}
Step 1 assumes that the relation embeddings $\{\mathbf{e}_r\}_{r \in \mathcal{R}}$ are linearly independent, which guarantees the existence of a dual basis enabling simultaneous concentration ($\beta_{r_t} \to 1$) and suppression ($\beta_{r''} \to 0$) of relation attention scores. In practice, this condition is generically satisfied for high-dimensional embeddings: a set of $|\mathcal{R}|$ vectors in $\mathbb{R}^d$ with $d \gg |\mathcal{R}|$ is almost surely linearly independent under any continuous distribution. Furthermore, recent LLM probing studies show that relational structures in high-dimensional dense representations---such as those produced by Qwen3-Embedding---tend to emerge as approximately linearly independent subspaces \cite{hernandez2024linearity, park2024linear}. The theorem should therefore be read as an \textit{architectural-capacity} guarantee: it shows that the RSF parameterization is expressive enough to realize the required relation alignment whenever the linear independence condition is met by the underlying embedding space.

\subsection{Proof of Theorem \ref{thm:consistency} (Guaranteed Structural Faithfulness)}
\label{app:proof_consistency}
\textbf{Statement:} Let $v_T$ be a target node at hop $T$ with non-zero structural flow, i.e., $\Phi_{v_T}^{(T)} > 0$. The greedy backtracking procedure, defined by $v_{t-1} = \operatorname{argmax}_{u \in \mathcal{N}_{in}(v_t)} \pi_u^{(t-1)}$, is guaranteed to reconstruct a continuous, valid path $P = (v_0, \dots, v_T)$ connecting the source entity $v_{topic}$ to $v_T$ within $T$ steps.

\begin{proof}
Recall the structural flow propagation equation for a node $v$ at hop $t$ in the RSF module:
\begin{equation}
    \Phi_v^{(t)} = \sum_{(u, r, v) \in \mathcal{E}} \pi_u^{(t-1)} \cdot \beta_r^{(t)}
\end{equation}
where $\pi_u^{(t-1)}$ is the probability mass at predecessor node $u$, and $\beta_r^{(t)} \in [0, 1]$ is the relation attention score produced by the Dynamic Relation Attention module (Eq.~\ref{eq:relatt}). The greedy backtracking algorithm selects the predecessor at step $t$ via:
\begin{equation}
    v_{t-1}^* = \underset{u \in \mathcal{N}_{in}(v_t)}{\operatorname{argmax}} \ \pi_u^{(t-1)}
\end{equation}
We assume the initialization $\boldsymbol{\pi}^{(0)}$ is a one-hot vector at the source node $v_{topic}$, such that $\pi_{v}^{(0)} = 1$ if $v=v_{topic}$ and $0$ otherwise.

We prove the theorem by induction on the hop index $t$, proceeding backwards from the target hop $T$ down to the source hop $0$.

\textit{Base Assumption (Step $T$):}
The premise of the theorem states that the selected target node $v_T$ has received non-zero structural flow: $\Phi_{v_T}^{(T)} > 0$.

\textit{Inductive Step:}
Consider an arbitrary node $v_t$ selected at hop $t$ (where $0 < t \le T$) such that its structural flow $\Phi_{v_t}^{(t)} > 0$.
By the definition of the flow summation, if $\Phi_{v_t}^{(t)} > 0$, the sum $\sum_{u} \pi_u^{(t-1)} \beta_r^{(t)}$ is strictly positive. Since probabilities $\pi$ and attention scores $\beta$ are non-negative, this implies that the set of contributing predecessors is non-empty. Specifically, there exists at least one neighbor $u' \in \mathcal{N}_{in}(v_t)$ such that:
\begin{equation}
    \pi_{u'}^{(t-1)} \cdot \beta_{r'}^{(t)} > 0 \implies \pi_{u'}^{(t-1)} > 0
\end{equation}
The greedy backtracking algorithm selects the predecessor $v_{t-1}^*$ that maximizes $\pi_u^{(t-1)}$. Since the set of neighbors with non-zero probability is non-empty (containing at least $u'$), the maximum value in this set must be strictly positive:
\begin{equation}
    \pi_{v_{t-1}^*}^{(t-1)} \ge \pi_{u'}^{(t-1)} > 0
\end{equation}
Thus, the selected predecessor $v_{t-1}^*$ has non-zero probability mass.
If $t-1 > 0$, this node $v_{t-1}^*$ must have received flow from its own predecessors (since $\pi^{(t-1)}$ is derived from the flow $\Phi^{(t-1)}$). Therefore, $\Phi_{v_{t-1}^*}^{(t-1)} > 0$.
This establishes the inductive step: if the node at $t$ has valid flow, the selected node at $t-1$ also has valid flow.

\textit{Termination (Step $0$):}
The recursion proceeds until we reach $t=0$, selecting a node $v_0$ with $\pi_{v_0}^{(0)} > 0$.
By definition, the initial distribution $\boldsymbol{\pi}^{(0)}$ is a one-hot vector peaked at $v_{topic}$. The only node with non-zero probability at $t=0$ is $v_{topic}$.
Therefore, it must be that $v_0 = v_{topic}$.

\textit{Conclusion:}
The sequence of nodes $P = (v_0, v_1, \dots, v_T)$ constructed by the backtracking procedure satisfies two conditions:
1. $v_0 = v_{topic}$.
2. For every step $t > 0$, $v_{t-1}$ is a valid neighbor of $v_t$ in $\mathcal{G}$ (specifically, the neighbor contributing maximal flow).

Consequently, $P$ forms a valid, continuous reasoning path in the Knowledge Graph starting at the source entity and ending at the answer candidate.
\end{proof}

\section{Related Work}
\label{sec:related_work}

We position our work within the broader landscape of Knowledge Graph Question Answering (KGQA), categorized into subgraph retrieval, LLM-integrated frameworks, and agentic reasoning.

\textbf{Subgraph Retrieval and Reasoning.}
Early neural approaches to KGQA formulated reasoning as a path traversal or subgraph retrieval problem. Methods like \textbf{GraftNet} \cite{sun2018open} and \textbf{PullNet} \cite{sun2019pullnet} employ a \textit{Retrieve-then-Read} paradigm, identifying a discrete subgraph via heuristic or learned retrieval before applying a GNN reader. However, the discrete selection of nodes breaks end-to-end differentiability, preventing the retriever from adapting to downstream reasoning errors. While \textbf{NSM} (Neural State Machine) \cite{he2021improving} introduced a differentiable instruction signal, it relies on fixed node embeddings, causing it to fail when intermediate \textit{bridge} nodes lack lexical overlap with the query (the \textit{Semantic Gap}). \textbf{EmbedKGQA} \cite{saxena-etal-2020-improving} mitigates this via latent space matching but lacks the explicit multi-hop traversal required for complex compositional queries.
\textit{Our Contribution:} RSF-GLLM addresses the non-differentiability of retrieval via continuous soft-flow propagation and resolves the Semantic Gap using a novel Dynamic Gating Mechanism that prioritizes structural validity over semantic matching when necessary.

\textbf{LLM-KG Integration.}
Recent paradigms leverage the generative power of LLMs to interpret graph structures. Frameworks like \textbf{RoG} (Reasoning on Graphs) \cite{luo2024reasoning} and \textbf{GNN-RAG} generate reasoning paths as planning steps or retrieve graph context to augment generation. While effective, these methods face two critical limitations: (1) \textbf{Computational Cost:} RoG requires fine-tuning large base models (e.g., LLaMA-2-7B), imposing significant memory and training overheads. (2) \textbf{Faithfulness:} As noted in recent studies \cite{gao2023retrieval, li2024chain}, retrieval-augmented LLMs are prone to \textit{reasoning shortcuts}, ignoring retrieved structures in favor of parametric memory.
\textit{Our Contribution:} We propose a decoupled architecture where reasoning is handled by a lightweight ($\sim$40ms) GNN module. This avoids the cost of LLM fine-tuning for reasoning and ensures grounding by conditioning the frozen LLM strictly on extracted paths.

\textbf{Agentic and Chain-of-Thought Reasoning.}
The state-of-the-art in complex reasoning involves treating LLMs as autonomous agents. Methods like \textbf{ToG} (Think-on-Graph) \cite{sun2024think} and \textbf{Graph of Thoughts} \cite{besta2024graph} perform iterative beam search or Monte Carlo Tree Search (MCTS) over the KG, invoking the LLM at every hop to evaluate candidates. While achieving high accuracy, these \textit{agentic} approaches are prohibitively slow, often requiring 10--50+ LLM calls per query \cite{sun2024think, besta2024graph}.
\textit{Our Contribution:} RSF-GLLM achieves comparable multi-hop reasoning depth with a \textit{single} LLM call. By offloading the iterative search to our Recurrent Soft-Flow module, we reduce inference latency by orders of magnitude compared to agentic baselines while maintaining theoretical guarantees on path sparsity and validity.

\section{Baselines}\label{sec:baslines}
We evaluate our proposed method against five categories of state-of-the-art baselines, ranging from pure parametric reasoning to specialized graph-agent frameworks.

\begin{enumerate}
    \item \textbf{Embedding-based Methods:} This category includes traditional neural approaches that learn to propagate signals or match questions in a latent embedding space. We compare against \textit{KV-Mem}, which uses key-value memory networks for reading documents; \textit{EmbedKGQA}, which leverages knowledge graph embeddings for link prediction; \textit{NSM} (Neural State Machine), which simulates sequential reasoning steps; and \textit{TransferNet}, which utilizes a differentiable framework to propagate attention scores across graph relations.
    
    \item \textbf{Graph Retrieval Methods:} These frameworks focus on identifying question-relevant subgraphs or using neural traversals to narrow the search space. This category includes \textit{GraftNet}, which uses static embedding matching; \textit{PullNet}, which iteratively retrieves relevant subgraphs; \textit{SR+NSM}, a stepwise neural reasoner with subgraph retrieval; \textit{ReaRev}, which adaptively revises reasoning instructions based on graph feedback; and \textit{UniKGQA}, a unified architecture for retrieval and reasoning.
    
    \item \textbf{LLM Reasoning:} These methods evaluate the parametric knowledge of Large Language Models (LLMs) without providing external Knowledge Graph (KG) context. We evaluate \textit{Qwen3-8B} and \textit{LLaMA3.1-8B} using zero-shot prompting to establish a baseline for pure LLM reasoning capability.
    
    \item \textbf{LLM + KG Integration:} These methods leverage the reasoning power of LLMs by augmenting their context with retrieved KG facts. We compare against \textit{KD-CoT}, an interactive framework for faithful reasoning; \textit{DECAF}, which jointly decodes answers and logical forms using Fusion-in-Decoder; \textit{RoG}, which generates grounded reasoning plans; \textit{G-Retriever}, a retrieval-augmented generation method for graphs; and \textit{GNN-RAG}, which uses Graph Neural Networks to retrieve paths for LLM reasoning.
    
    \item \textbf{Agentic Search Frameworks:} The most recent frontier treats the LLM as an autonomous agent that explores the KG through sequential decision-making. This includes \textit{EffiQA}, which employs strategic multi-model collaboration for efficient QA, and \textit{FD-PORT}, which utilizes Monte Carlo Tree Search (MCTS) and flow-guided optimization to navigate complex multi-hop reasoning paths.
\end{enumerate}

\section{Dynamic Gating Analysis}
\label{sec:lambda_analysis}

To empirically validate the semantic gap bridging hypothesis (Theorem~\ref{thm:expressivity}), we analyze the evolution of the dynamic gating weight $\lambda^{(t)}$ across reasoning hops. Recall that $\lambda^{(t)} \to 0$ indicates structure-focused propagation (traversing via graph topology), while $\lambda^{(t)} \to 1$ indicates content-focused selection (matching node semantics to the query).

Figure~\ref{fig:lambda_evolution} presents the mean $\lambda^{(t)}$ values aggregated across all 4-hop questions in the CWQ test set. The results reveal an interesting pattern that validates our theoretical framework.

\begin{figure}[h]
    \centering
    \includegraphics[width=0.6\textwidth]{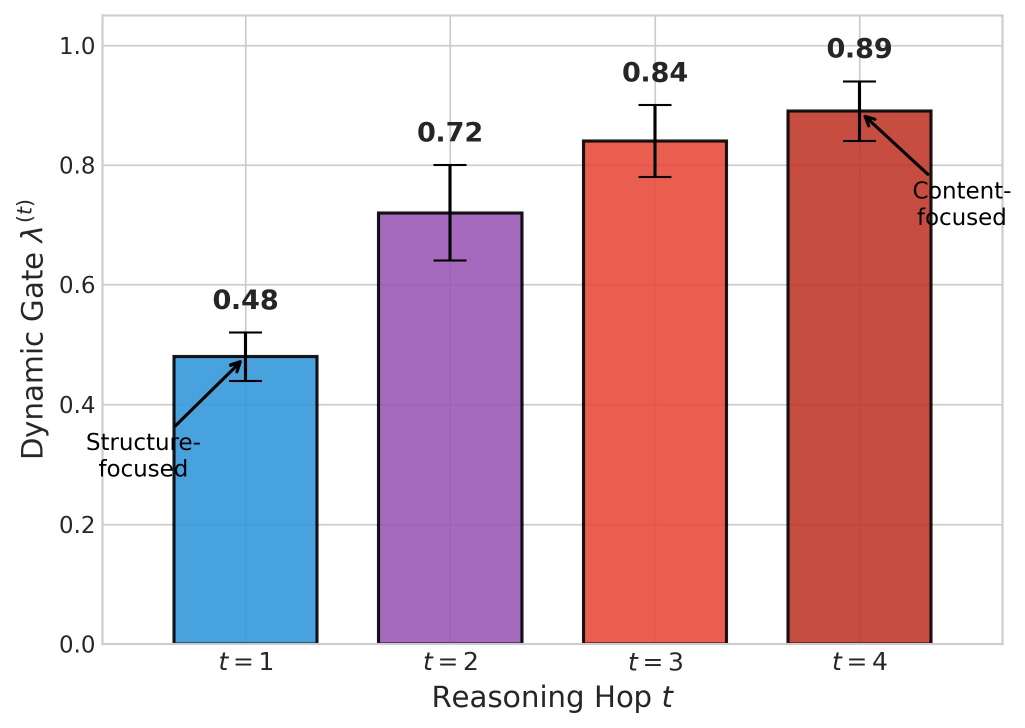}
    \caption{Evolution of dynamic gating weight $\lambda^{(t)}$ across reasoning hops on CWQ 4-hop questions. Error bars indicate standard deviation. The monotonic increase from structure-focused ($\lambda \approx 0.48$) to content-focused ($\lambda \approx 0.89$) validates the semantic gap bridging hypothesis.}
    \label{fig:lambda_evolution}
\end{figure}

\textbf{Interpretation.} The learned gating schedule demonstrates an adaptive transition from structure-focused to content-focused reasoning. At earlier hops, $\lambda$ remains low ($\approx 0.48$), indicating that the model prioritizes graph topology over semantic matching. This is precisely the behavior required to bridge the semantic gap: when traversing from a topic entity like ``Inception'' to an intermediate node like ``Christopher Nolan,'' lexical similarity to query terms such as ``awards'' would be misleading, so the model relies on valid relation types to guide propagation. As reasoning progresses, $\lambda$ increases steadily, reflecting the model's growing confidence in using content signals for disambiguation---particularly important in fan-out scenarios where multiple neighbors share the same valid relation. At later hops ($\lambda \approx 0.89$), content matching dominates, ensuring the final answer aligns semantically with the original query intent. This monotonic increase emerges naturally from end-to-end training without explicit supervision on $\lambda$ values, demonstrating that RSF learns to adaptively balance structure and semantics based on the reasoning stage.

We also remark that individual case by case analysis reveals that we notice samples with $\lambda$ $<$ 0.5 in the final hop, notably when the answer node itself is a generic entity (eg.\ ``National Language''). 

\textbf{CWQ Gate Ablation.} To isolate the dynamic gate's effect on the harder CWQ benchmark, we trained RSF from scratch without the dynamic gate ($\lambda \equiv 1$). Table~\ref{tab:cwq_gate_ablation} reports the results. The gate provides a consistent benefit across all retrieval metrics, with the largest impact on H@1 ($-4.0$pt), substantially larger than the $-2.4$pt H@10 drop observed on WebQSP (Table~\ref{tab:ablation}). This confirms that CWQ's more complex compositional queries expose the gate's value more clearly. We note that CWQ does not provide explicit hop-depth annotations---questions are categorized by composition type, not reasoning depth---so a hop-stratified breakdown is not available from the dataset.

\begin{table}[h]
    \centering
    \caption{Dynamic gate ablation on CWQ retrieval. Removing the gate ($\lambda \equiv 1$) causes a larger H@1 drop ($-4.0$pt) than on WebQSP ($-2.4$pt H@10, Table~\ref{tab:ablation}), confirming the gate's value scales with query complexity.}
    \label{tab:cwq_gate_ablation}
    \small
    \begin{tabular}{l|ccc}
        \toprule
        \textbf{Model} & \textbf{H@1} & \textbf{H@5} & \textbf{H@10} \\
        \midrule
        RSF (with gate)       & 51.3 & 67.7 & 72.1 \\
        RSF (w/o gate, $\lambda{=}1$) & 47.3 & 64.6 & 69.8 \\
        $\Delta$              & $-4.0$ & $-3.1$ & $-2.3$ \\
        \bottomrule
    \end{tabular}
\end{table}

\section{Qualitative Analysis of Reasoning Chains}
\label{sec:qualitative}

To provide deeper insight into RSF-GLLM's reasoning behavior, we present representative examples of successful and failed predictions from the CWQ and WebQSP test sets. These examples illustrate the model's ability to perform complex multi-hop reasoning while also revealing systematic failure modes.

\subsection{Successful Multi-Hop Reasoning}

The following examples demonstrate RSF-GLLM's ability to correctly navigate 3--4 hop reasoning chains, including cases requiring traversal through semantically dissimilar intermediate nodes.

\begin{examplebox}{Example 1: Geographic $\rightarrow$ Political Inference (CWQ, 4-hop)}
\textbf{Question:} ``What type of government is used in the country with Northern District?''

\textbf{Topic Entity:} Northern District \quad \textbf{Ground Truth:} Parliamentary System

\textbf{Prediction:} Parliamentary System \quad \textbf{F1:} 1.00

\begin{pathbox}
Northern District $\xrightarrow{\texttt{administrative\_division.country}}$ Israel 
  $\xrightarrow{\texttt{governmental\_jurisdiction.governing\_officials}}$ [position\_m.0j\_m3\_b]
  $\xrightarrow{\texttt{government\_position\_held.jurisdiction\_of\_office}}$ Israel
  $\xrightarrow{\texttt{location.country.form\_of\_government}}$ Parliamentary System
\end{pathbox}

\textbf{Analysis:} This example demonstrates clean 4-hop reasoning through geographic and political knowledge. Starting from an obscure administrative district, the model correctly identifies Israel as the parent country, navigates through government official positions (CVT nodes), and extracts the governmental structure. 
\end{examplebox}

\begin{examplebox}{Example 2: Multi-Answer Language Retrieval (CWQ, 4-hop)}
\textbf{Question:} ``The people from the country that contains Nord-Ouest Department speak what languages today?''

\textbf{Topic Entity:} Nord-Ouest Department \quad \textbf{Ground Truth:} Haitian Creole $|$ French

\textbf{Prediction:} French, Haitian Creole \quad \textbf{F1:} 1.00

\begin{pathbox}
Path 1: Nord-Ouest Department $\xrightarrow{\texttt{country}}$ Haiti 
  $\xrightarrow{\texttt{governing\_officials}}$ [m.010g48js]
  $\xrightarrow{\texttt{jurisdiction\_of\_office}}$ Haiti
  $\xrightarrow{\texttt{official\_language}}$ French

Path 2: Nord-Ouest Department $\xrightarrow{\texttt{country}}$ Haiti 
  $\xrightarrow{\texttt{governing\_officials}}$ [m.010g48js]
  $\xrightarrow{\texttt{jurisdiction\_of\_office}}$ Haiti
  $\xrightarrow{\texttt{official\_language}}$ Haitian Creole
\end{pathbox}

\textbf{Analysis:} This example showcases RSF's ability to retrieve \textit{multiple correct answers} through parallel reasoning paths. Both official languages of Haiti are successfully identified via paths that share the first three hops but diverge at the final relation. The flow sparsity regularization allows the model to maintain multiple high-probability candidates when the question semantics warrant it, rather than collapsing to a single answer prematurely.
\end{examplebox}

\begin{examplebox}{Example 3: Historical Multi-Answer Question (WebQSP, 3-hop)}
\textbf{Question:} ``What else did Ben Franklin invent?''

\textbf{Topic Entity:} Benjamin Franklin \quad \textbf{Ground Truth:} Lightning rod $|$ Glass harmonica $|$ Bifocals $|$ Franklin stove

\textbf{Prediction:} Franklin stove, Glass harmonica, Bifocals, Lightning rod \quad \textbf{F1:} 1.00

\begin{pathbox}
Benjamin Franklin $\xrightarrow{\texttt{image.appears\_in\_topic}}$ Benjamin Franklin
  $\xrightarrow{\texttt{image.appears\_in\_topic}}$ Benjamin Franklin
  $\xrightarrow{\texttt{law.inventor.inventions}}$ [All 4 inventions retrieved]
\end{pathbox}

\textbf{Analysis:} Despite the seemingly redundant intermediate hops through image relations, RSF correctly identifies all four inventions. This demonstrates robustness to noisy paths in the KG---the model learns that certain relation types (e.g., \texttt{image.appears\_in\_topic}) act as identity connectors and successfully propagates flow through them to reach the target \texttt{inventions} relation.
\end{examplebox}

\begin{examplebox}{Example 4: Political Succession (WebQSP, 3-hop)}
\textbf{Question:} ``Who was vice president after Kennedy died?''

\textbf{Topic Entity:} John F. Kennedy \quad \textbf{Ground Truth:} Lyndon B. Johnson

\textbf{Prediction:} Lyndon B. Johnson \quad \textbf{F1:} 1.00

\begin{pathbox}
John F. Kennedy $\xrightarrow{\texttt{common.topic.webpage}}$ [m.0b48xzc]
  $\xrightarrow{\texttt{common.webpage.topic}}$ John F. Kennedy
  $\xrightarrow{\texttt{government.us\_president.vice\_president}}$ Lyndon B. Johnson
\end{pathbox}

\textbf{Analysis:} The question contains implicit temporal reasoning (``after Kennedy died''), which the model handles by recognizing that JFK's vice president at the time of death would succeed him. The intermediate webpage node serves as a structural bridge, and the final hop directly extracts the vice-presidential relationship.
\end{examplebox}

\begin{examplebox}{Example 5: Implicit Religious Context (CWQ, 4-hop)}
\textbf{Question:} ``What is the predominant religion where the leader is Ovadia Yosef?''

\textbf{Topic Entity:} Ovadia Yosef \quad \textbf{Ground Truth:} Judaism

\textbf{Prediction:} Judaism \quad \textbf{F1:} 1.00

\begin{pathbox}
Ovadia Yosef $\xrightarrow{\texttt{STAY}}$ Ovadia Yosef
  $\xrightarrow{\texttt{people.person.religion}}$ Judaism
  $\xrightarrow{\texttt{religion.founding\_figures}}$ Abraham
  $\xrightarrow{\texttt{religion.founding\_figure.religion\_founded}}$ Judaism
\end{pathbox}

\textbf{Analysis:} This example demonstrates handling of implicit constraints. The phrase ``where the leader is'' does not explicitly mention religious leadership, yet the model correctly infers that Ovadia Yosef's domain of leadership is religious. The reasoning path exhibits a cyclic structure (Judaism $\rightarrow$ Abraham $\rightarrow$ Judaism), showing that RSF can handle re-entrant paths as discussed in Section 2.2.
\end{examplebox}

\subsection{Failure Mode Analysis}

We identify three systematic failure modes through analysis of incorrect predictions. Understanding these limitations is crucial for guiding future improvements.

\begin{examplebox}{Failure Mode 1: Answer Entity Not in Subgraph (WebQSP)}
\textbf{Question:} ``Who plays Ken Barlow in Coronation Street?''

\textbf{Topic Entity:} Coronation Street \quad \textbf{Ground Truth:} William Roache

\textbf{Prediction:} \textit{``The provided reasoning paths do not contain information about the actor...''} \quad \textbf{F1:} 0.00

\textbf{Subgraph Reachability:} No --- Answer entity not present in extracted subgraph.

\begin{pathbox}
Path 1: Coronation Street $\xrightarrow{\texttt{program\_creator}}$ Tony Warren 
  $\xrightarrow{\texttt{film.actor.film}}$ [m.0h0\_mvx] $\xrightarrow{\texttt{performance.actor}}$ Tony Warren

Path 8: Coronation Street $\xrightarrow{\texttt{program\_creator}}$ Tony Warren
  $\xrightarrow{\texttt{fictional\_characters\_created}}$ Ken Barlow $\xrightarrow{\texttt{STAY}}$ Ken Barlow
\end{pathbox}

\textbf{Analysis:} The answer entity \textbf{William Roache} is not present in the extracted subgraph. While the character ``Ken Barlow'' appears in Path 8, the actor-character relationship lies beyond the $K$-hop extraction radius. 
\end{examplebox}

\begin{examplebox}{Failure Mode 2: LLM Selects Wrong Candidate from Multiple Options (CWQ)}
\textbf{Question:} ``Which man is the leader of the country that uses `Libya, Libya, Libya' as its national anthem?''

\textbf{Topic Entity:} Prime Minister of Libya \quad \textbf{Ground Truth:} Abdullah al-Thani

\textbf{Prediction:} Muammar Gaddafi \quad \textbf{F1:} 0.00

\textbf{Subgraph Reachability:} Yes --- Correct answer is present in reasoning paths.

\begin{pathbox}
Path 1: Prime Minister of Libya $\xrightarrow{\texttt{office\_holders}}$ [m.011n4mp5]
  $\xrightarrow{\texttt{office\_holder}}$ Abdullah al-Thani $\xrightarrow{\texttt{positions\_held}}$ ...

Path 3: Prime Minister of Libya $\xrightarrow{\texttt{office\_holders}}$ [m.07lt4tb]
  $\xrightarrow{\texttt{office\_holder}}$ Muammar Gaddafi $\xrightarrow{\texttt{positions\_held}}$ ...

Path 4: Prime Minister of Libya $\xrightarrow{\texttt{office\_holders}}$ [m.0h89fd8]
  $\xrightarrow{\texttt{office\_holder}}$ Mahmoud Jibril $\xrightarrow{\texttt{positions\_held}}$ ...
\end{pathbox}

\textbf{Analysis:} The reasoning paths correctly contain the ground truth answer (Abdullah al-Thani in Path 1), but also contain multiple other Libyan leaders. The LLM incorrectly selected Gaddafi due to popularity bias in its parametric knowledge. The KG lacks temporal annotations to disambiguate which leader is ``current.''
\end{examplebox}

\begin{examplebox}{Failure Mode 3: Entity Alias Mismatch (CWQ)}
\textbf{Question:} ``What was the name of the team that won the 2008 FIFA Club World Cup Final championship?''

\textbf{Topic Entity:} 2008 FIFA Club World Cup Final \quad \textbf{Ground Truth:} Newton Heath L\&YR F.C.

\textbf{Prediction:} Manchester United F.C. \quad \textbf{F1:} 0.00

\begin{pathbox}
Path 1: 2008 FIFA Club World Cup Final $\xrightarrow{\texttt{sports\_championship\_event.champion}}$ 
  Manchester United F.C. $\xrightarrow{\texttt{STAY}}$ Manchester United F.C. ...

Path 2: 2008 FIFA Club World Cup Final $\xrightarrow{\texttt{champion}}$ Manchester United F.C.
  $\xrightarrow{\texttt{sports\_team.venue}}$ [m.0n4ykf9] $\xrightarrow{\texttt{team}}$ Manchester United F.C.
\end{pathbox}

\textbf{Analysis:} All reasoning paths correctly lead to ``Manchester United F.C.,'' and the LLM correctly outputs this entity. However, the ground truth annotation uses ``Newton Heath L\&YR F.C.''---the historical name of Manchester United from 1878--1902. This is a dataset annotation artifact rather than a model failure; the prediction is factually correct but receives F1 of 0.00 due to string mismatch with an outdated alias. To mitigate such issues in production environments, we propose integrating an alias aware entity linker that normalizes both KG output and ground truth annotations to a unique uuid, ensuring the model is not penalized for nomenclature differences.
\end{examplebox}

\subsection{Summary of Failure Modes}

Table~\ref{tab:failure_modes} summarizes the three identified failure modes with their characteristics.

\begin{table}[h]
    \centering
    \caption{Summary of systematic failure modes identified in qualitative analysis.}
    \label{tab:failure_modes}
    \small
    \begin{tabular}{p{3.5cm}|p{4.5cm}|p{2.5cm}}
        \toprule
        \textbf{Failure Mode} & \textbf{Root Cause} & \textbf{Answer in Paths?} \\
        \midrule
        Answer Not in Subgraph & Answer entity lies outside $K$-hop subgraph & No \\
        \midrule
        Wrong Candidate Selection & LLM popularity bias; no temporal info in KG & Yes \\
        \midrule
        Entity Alias Mismatch & Dataset annotation uses outdated entity name & N/A (correct) \\
        \bottomrule
    \end{tabular}
\end{table}

These qualitative examples validate that RSF-GLLM successfully performs complex multi-hop reasoning across diverse question types, while also revealing that remaining errors often stem from retrieval coverage limitations or LLM-side biases rather than fundamental reasoning failures in the RSF module.

\section{Implementation Details and Reproducibility}
\label{sec:implementation}

We provide comprehensive implementation details to facilitate reproducibility of our results.

\subsection{Compute Resources}

All experiments were conducted on a single NVIDIA A100 GPU with 80GB HBM2e memory. No distributed training or multi-GPU setups were required, demonstrating the efficiency of our decoupled architecture.

\subsection{Stage 1: RSF Module Training}

\textbf{Architecture.} The RSF module consists of: (i) a projection layer mapping Qwen3 embeddings to lower hidden states, (ii) a GRU for query updating, (iii) attention layers for relation scoring  and content bias, and (iv) a gating MLP with a single linear layer followed by sigmoid activation.

\textbf{Embeddings.} Node and relation embeddings are obtained using Qwen3-Embedding. Embeddings are pre-computed and cached for all entities and relations in the Freebase subset used by WebQSP and CWQ. Entity names and relation texts are used as input text.

\textbf{Optimization.} We use AdamW optimizer with learning rate $5 \times 10^{-5}$, weight decay $0.01$, and linear warmup over the first 10\% of training steps. Training runs for 10 epochs with early stopping based on validation Hit@1 (patience = 3 epochs). Batch size is 16 for WebQSP and 8 for CWQ (due to larger subgraphs). Gradient clipping is applied with max norm 1.0.

\textbf{Hyperparameters.} The flow sparsity coefficient $\lambda_1 = 0.1$ was selected via grid search over $\{0.01, 0.05, 0.1, 0.2, 0.5\}$. The number of reasoning hops $T$ is set to 2 for WebQSP and 4 for CWQ, matching the dataset characteristics. The smoothing constant $\epsilon = 10^{-8}$ is used for numerical stability.

\textbf{Subgraph Extraction.} We perform $K$-hop BFS from the topic entity ($K=2$ for WebQSP, $K=4$ for CWQ). 

\subsection{Stage 2: LLM Fine-Tuning}

\textbf{Base Model.} We report results with two answer-generation backbones: \textbf{LLaMA-2-7B} and \textbf{Qwen3-8B}. LLaMA-2-7B is included to enable an apples-to-apples comparison with prior LLM+KG baselines (e.g., RoG, GNN-RAG, GNN-RAG+RA), which use the same backbone. Qwen3-8B is reported as a stronger contemporary alternative of comparable parameter scale. Both backbones are trained and evaluated under identical hyperparameters and the same RSF-extracted reasoning paths; no other component of the pipeline changes between the two runs.

\textbf{Fine-Tuning.} We apply full parameter fine-tuning for both backbones, using a single combined checkpoint evaluated on both WebQSP and CWQ to match the protocol used by the LLaMA-2-7B baselines.

\textbf{Optimization.} We fine-tune using AdamW with learning rate $2 \times 10^{-4}$, cosine learning rate schedule, and batch size 1. Training runs for 3 epochs. Maximum sequence length is 2048 tokens.

\subsection{Evaluation Metrics}

We follow standard KGQA evaluation protocols. \textbf{Hit@1} measures exact match accuracy---whether the top-1 predicted answer matches any gold answer. \textbf{F1} computes token-level overlap between the predicted answer string and gold answers, averaged across all test samples.

%%%%%%%%%%%%%%%%%%%%%%%%%%%%%%%%%%%%%%%%%%%%%%%%%%%%%%%%%%%%%%%%%%%%%%%%%%%%%%%
%%%%%%%%%%%%%%%%%%%%%%%%%%%%%%%%%%%%%%%%%%%%%%%%%%%%%%%%%%%%%%%%%%%%%%%%%%%%%%%

\end{document}